\documentclass{article}

\usepackage[preprint]{neurips_2026}
\makeatletter
\renewcommand{\@noticestring}{Accepted at the 40th Conference on Neural Information Processing Systems (NeurIPS 2026).}
\makeatother

\usepackage[utf8]{inputenc} 
\usepackage[T1]{fontenc}    
\usepackage{microtype}
\usepackage{graphicx}
\usepackage{subcaption}
\usepackage{booktabs} 
\usepackage{multirow}
\usepackage{wrapfig}
\usepackage{caption}
\usepackage{hyperref}
\usepackage{url}
\usepackage{algorithm}
\usepackage{algorithmic}

\usepackage{amsmath}
\usepackage{amssymb}
\usepackage{amsfonts}
\usepackage{mathtools}
\usepackage{amsthm}

\usepackage[capitalize,noabbrev]{cleveref}

\theoremstyle{plain}
\newtheorem{theorem}{Theorem}[section]
\newtheorem{proposition}[theorem]{Proposition}

\theoremstyle{definition}

\theoremstyle{remark}

\usepackage[disable,textsize=tiny]{todonotes}

\usepackage{tabularx}
\usepackage[table]{xcolor}
\usepackage[most]{tcolorbox}
\usepackage{enumitem}
\usepackage{xurl}
\usepackage{nicefrac}
\usepackage{xspace}

\usepackage{tikz}

\newcommand{\revision}[1]{#1}

\newcommand{\finalrev}[1]{#1}
\newcommand{\rebrev}[1]{#1}

\definecolor{darkgreen}{rgb}{0.0, 0.5, 0.0}

\newcommand{\passat}[1]{\text{pass@}#1}
\newcommand{\Passat}[1]{\text{Pass@}#1}

\newcommand{\zunit}{z3\_unit\xspace}
\newcommand{\zternary}{z3\_ternary\xspace}
\newcommand{\zheuleschur}{z3\_heule\_schur\xspace}
\newcommand{\zheuleu}{z3\_heuleu\xspace}
\newcommand{\zmarchcu}{z3\_march\_cu\xspace}

\title{Learning How to Cube}

\author{%
  Ferhat Erata$^{1,2}$\thanks{Equal contribution.} \quad
  Sam Kouteili$^{1}$\footnotemark[1] \quad
  Thanos Typaldos$^{1}$\footnotemark[1] \\[0.5em]
  \textbf{Timos Antonopoulos}$^{1}$ \quad
  \textbf{Robert B. Jones}$^{2}$ \quad
  \textbf{Byron Cook}$^{2}$ \quad
  \textbf{Ruzica Piskac}$^{1}$ \\[0.8em]
  $^{1}$Yale University \quad $^{2}$Amazon Web Services
}

\begin{document}

\maketitle

\begin{abstract}
    Despite the effectiveness of Cube-and-Conquer (C\&C) for solving challenging Boolean Satisfiability (SAT) problems, no prior work has shown that transformer-based models can learn effective cubing heuristics. We introduce a neuro-symbolic post-training framework for this task. We design an MCTS-based data curation pipeline that uses symbolic heuristics to explore splitting decisions over SAT competition formulas, producing preference data grounded in solver statistics and augmented with reasoning traces from a teacher model. Our two-stage post-training, supervised fine-tuning (SFT) followed by direct preference optimization (DPO), enables a 4B-parameter model to achieve a \passat{5} score of 53 on 100 SAT competition benchmarks, surpassing frontier LLMs such as Claude-Sonnet-4 (50) and matching\rebrev{, with five attempts per benchmark,} the best symbolic heuristic (53)\rebrev{, which nonetheless remains preferable for a single attempt}. Ablations show that SFT alone improves \passat{5} from 46 to 51, with DPO adding 2 additional benchmarks; an entropy/agreement ablation on realized first-cube decisions further shows that SFT, not DPO, accounts for the root-level decision diversity that produces complementary per-run coverage over deterministic symbolic methods. This demonstrates that transformers can be trained to make effective cubing decisions in a domain traditionally dominated by symbolic methods.
\end{abstract}

\section{Introduction}\label{sec:intro}

Propositional satisfiability (SAT) is a foundational computational problem with widespread applications in formal verification, program analysis, planning, and artificial intelligence~\citep{kautz1992planning,prasad2005survey, ye2023satlm}. Given a Boolean formula in conjunctive normal form (CNF), the SAT problem asks whether there exists a variable assignment that satisfies all clauses. This problem is NP-complete, with several heuristic-guided algorithms devised to solve problem classes~\citep{moskewicz2001chaff}.

\begin{figure*}[t]
    \centering
    \begin{minipage}[t]{0.44\linewidth}
        \vspace{0pt}
        \includegraphics[width=0.95\linewidth]{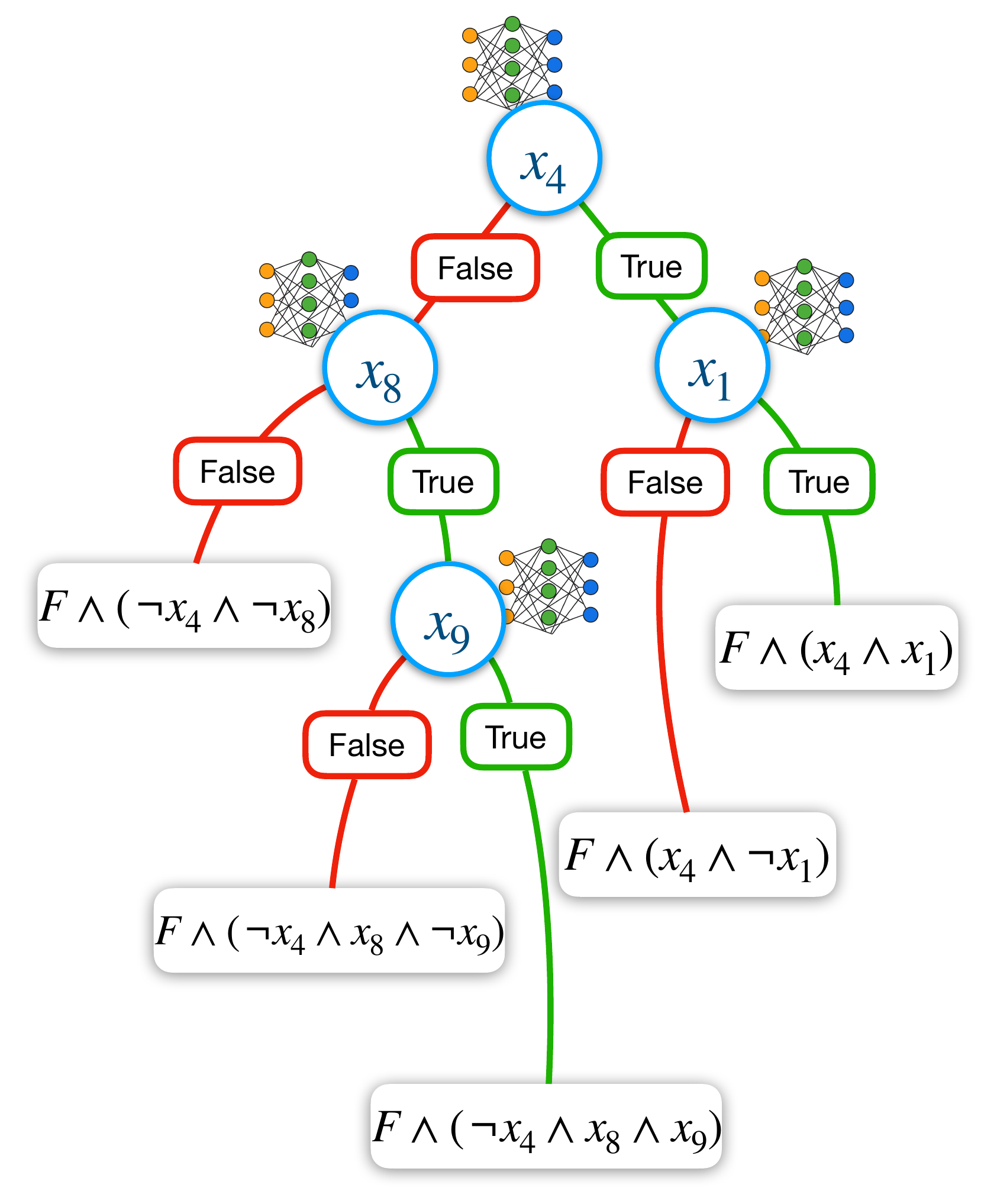}
        \label{fig:neural-heuristic}
    \end{minipage}
    \hfill
    \begin{minipage}[t]{0.55\textwidth}
        \vspace{-9pt}
        \scalebox{0.88}{%
        \begin{minipage}{1.1\textwidth}
        \begin{algorithm}[H]
        \caption{Cube-and-Conquer}
        \label{alg:cube_and_conquer}
        \begin{algorithmic}[1]
            \STATE {\bfseries Input:} {node $n$, heuristic $heur$}
            \STATE {\bfseries Output:} {Run Statistics}
            \STATE $(v, \neg v) \gets$ \textsc{Choose-Split}($n,\, heur$) \\
            \STATE $(cn_1, cn_2) \gets$ \textsc{Create-Children}($n,\, (v, \neg v)$) \\
            \STATE $stats_1 \gets$ \textsc{Rollout($cn_1$)} \COMMENT{\textit{Process child 1}} \\
            \IF{$stats_1.sat\_status$ = \small{UNKNOWN}}
                \STATE $cube\_stats_1$ $\gets$ \textsc{Cube-and-Conquer}($cn_1,\, heur$) \\
                \STATE \small{merge}($stats_1,\, cube\_stats_1$)
            \ENDIF
            \STATE \textsc{Return-If-Timeout}($cube\_stats\_1$)
            \STATE $stats_2 \gets \emptyset$ \COMMENT{\textit{Process child 2}} \\
            \IF{$stats_1.sat\_status \neq$ \small{SAT}}
                \STATE $stats_2 \gets$ \textsc{Rollout($cn_2$)} \\
                \IF{$stats_2.sat\_status$ = \small{UNKNOWN}}
                    \STATE $cube\_stats_2$ $\gets$ \textsc{Cube-and-Conquer}($cn_2, heur$) \\
                    \STATE \small{merge}($stats_2,\, cube\_stats_2$)
                \ENDIF
            \ENDIF
            \STATE $cube\_stats \gets$ \small{combine}($stats_1,\, stats_2$) \COMMENT{\textit{Combine stats}}  \\
            \STATE \textbf{return} $cube\_stats$
        \end{algorithmic}
        \end{algorithm}
        \end{minipage}
        }
    \end{minipage}
    \caption{(a) Schematic depiction of Cube-and-Conquer using a neural heuristic, (b) Our implementation of the evaluation algorithm with 30 minutes timeout. The cubing phase is represented by \textsc{Choose-Split} and the conquering phase by \textsc{Rollout}, which is trying to solve the node's formula using a CDCL solver with 5 seconds timeout.} \label{fig:neural-heuristic-and-cube-and-conquer}
\end{figure*}

The Cube-and-Conquer (C\&C) method has emerged as a prominent decision procedure particularly suited for complex SAT problem instances~\citep{heule2011cube}. Cube-and-Conquer divides the SAT-solving process into two phases, namely a \emph{cubing phase}, which partitions the original formula into independent subproblems, or ``cubes'', followed by a \emph{conquer phase}, wherein each cube is solved individually with parallel Conflict-Driven Clause Learning (CDCL) solvers~\citep{biere2009handbook}.
The effectiveness of C\&C critically depends on the quality of the cubing heuristic, with a good heuristic producing cubes that simplify each subproblem and minimize redundant computational effort.
\finalrev{Prior work on learned SAT heuristics has focused on CDCL-style branching~\citep{liang2016learning,liang2017empirical} rather than C\&C cubing, and no transformer-based architecture has been shown to learn effective cubing heuristics.}
In this work, we propose a {\it neuro-symbolic} post-training framework that enables transformers to choose splitting variables for Cube-and-Conquer. \revision{\Cref{fig:neural-heuristic-and-cube-and-conquer} shows the schematic alongside our algorithm for dynamic cubing, in which the model's decisions define the cubes that are subsequently solved during the conquering phase.}
\revision{We first apply Supervised Fine Tuning (SFT) with teacher-generated data, then refine via Direct Preference Optimization (DPO)~\citep{rafailov2023direct} on preference pairs produced by Monte Carlo Tree Search (MCTS) over SAT competition problems, augmented with teacher reasoning traces.}

We evaluate on 100 held-out hard SAT competition benchmarks against state-of-the-art symbolic heuristics and frontier LLM baselines. \revision{Our 4B-parameter post-trained model surpasses frontier LLMs and matches the best symbolic heuristic at \passat{5} of 53 \rebrev{when each heuristic is given five attempts per benchmark. On a single run the \zunit heuristic remains preferable, solving 52 benchmarks against our 46 at \passat{1}}. Ablations show SFT raises the base model from 46 to 51 and DPO reaches 53, illustrating how neural and symbolic signals complement each other.}

\revision{Our contributions are
(1)~The first neuro-symbolic post-training framework enabling transformers to learn cubing heuristics for C\&C, achieving performance competitive with the best symbolic heuristics,
(2)~An MCTS-based data curation pipeline producing preference data grounded in solver statistics, with a novel bounded reward function and reasoning traces distilled from a teacher model,
(3)~A comprehensive evaluation on 100 SAT competition benchmarks with ablations isolating SFT and DPO contributions and analysis of per-run solution diversity, and
(4)~\finalrev{A mechanistic analysis showing that SFT, not DPO, accounts for the root-level decision diversity gap that underlies our model's \passat{5} advantage over deterministic symbolic heuristics.}}

\section{Related Work}\label{sec:related_work}

Several works have explored neural approaches to SAT solving. NeuroSAT~\citep{selsam2018learning} demonstrated that graph neural networks could learn satisfiability patterns, while NeuroCore~\citep{selsam2019guiding} used unsat-core predictions to guide CDCL solvers. Learning-based branching heuristics have shown promise, with LRB~\citep{liang2016learning} and its variants~\citep{liang2017empirical} using reinforcement learning to adapt branching decisions online. More recently, Monte Carlo Forest Search (MCFS)~\citep{cameron2024unsat} has been applied to synthesize UNSAT solver heuristics, and \finalrev{AlphaMapleSAT~\citep{jha2024alphamaplesat} used MCTS-based lookahead for cubing in C\&C, though with a non-transformer policy.} \rebrev{We evaluate it as an online-search (OS) baseline in our C\&C loop, where it solves more instances than any heuristic we test (56 per run against our 53) by running a full per-instance MCTS search at solve time. We report the full comparison separately in \cref{subsec:alphamaplesat} and keep it out of our main per-heuristic evaluation, which compares policies applied uniformly across instances.} Although transformers have shown reasoning capabilities on mathematical problems~\citep{hendrycks2021measuring}, direct investigation of transformer-based SAT reasoning~\citep{pan2024can} found significant limitations. No prior work has demonstrated that transformer-based architectures can learn effective cubing heuristics for C\&C. \finalrev{In a parallel line of work, \citet{zhai2025learning} apply deep reinforcement learning to splitting heuristics in Divide-and-Conquer SAT using a GNN policy trained from scratch. Our approach differs in architecture (transformer vs GNN), paradigm (post-training on preference data vs on-policy RL), and cubing style (decisions interleaved with solver rollouts vs a fixed tree built up front).}
\finalrev{Complementary to single-heuristic approaches, portfolio SAT solvers~\citep{xu2008satzilla,zhang2024grass} combine multiple solvers with different strengths. Our run-to-run diversity (\cref{sec:results}) makes the trained model a natural candidate for such portfolios.}

Our training methodology builds on recent advances in RL for LLMs. Direct Preference Optimization (DPO)~\citep{rafailov2023direct} has emerged as an effective alternative to RLHF for aligning model outputs with preferences. MCTS has been successfully applied to enhance LLM reasoning in domains like mathematics~\citep{zhang2025llama} and code generation. The combination of MCTS with preference learning has shown particular promise for problems requiring strategic lookahead~\citep{snell2024scaling}. We adapt these techniques to the SAT domain, using MCTS to explore the space of cubing decisions and generate preference data with a novel reward grounded in solver statistics.

\section{Preliminaries}\label{sec:background}

\noindent\textbf{Satisfiability}
The Boolean satisfiability problem (SAT) asks whether a given propositional formula admits a variable assignment of boolean values that evaluates it to true. We consider formulas expressed in conjunctive normal form (CNF), noting this as the standard in satisfiability solving domains. A CNF formula $\varphi$ is defined as a conjunction of clauses, each clause being a disjunction of literals. A literal $\ell$ is either a propositional variable $x$ (a positive literal) or its negation $\neg x$ (a negative literal). An assignment $\alpha$ maps each variable to a Boolean value $\{0,1\}$, corresponding to \emph{false} and \emph{true}. A positive literal $x$ is satisfied under $\alpha$ if $\alpha(x) = 1$, while a negative literal $\neg x$ is satisfied if $\alpha(x) = 0$. Conversely, a literal is falsified when these conditions are not met. A clause is satisfied if at least one of its literals is satisfied, and a CNF formula is satisfied if every clause is satisfied. A formula is \emph{satisfiable} if such an assignment exists, and \emph{unsatisfiable} otherwise.

\textbf{Conflict-Driven Clause Learning.}
Modern SAT solvers are typically based on Conflict-Driven Clause Learning (CDCL)~\citep{moskewicz2001chaff}. CDCL incrementally builds a partial assignment by branching on variables, then applying unit propagation to enforce implied assignments. If a conflicting assignment is detected, the solver analyzes it to derive a new clause that guides future search. This loop continues until either a satisfying assignment is found or a contradiction proves the formula unsatisfiable.

\textbf{Cube-and-Conquer.} Cube-and-Conquer (C\&C) is a SAT-solving paradigm designed for solving  complex instances~\citep{heule2011cube}. C\&C solvers introduce a partitioning step that breaks a SAT problem into smaller subproblems. The partitioning is achieved by selecting a sequence of \emph{splitting variables}, whose assignments define disjoint subformulas, or \emph{cubes}. Each splitting variable constructs two sub-formulas, simplifying for the positive and negative literal case. Each cube can then be solved independently by a conventional CDCL solver, typically in parallel. The efficiency of C\&C depends almost entirely on the quality of these splitting choices, as good variables yield balanced and simplified cubes, while poor choices lead to redundant or disproportionately hard subproblems. \finalrev{Designing effective cubing heuristics is a central challenge for C\&C, and a wide range of symbolic heuristics have been proposed for this task (formal definitions in \cref{subsec:symbolic-heuristics}), but they remain hand-designed.}
\section{Methodology}\label{sec:training-methodology}

\finalrev{Our downstream objective is to train a small language model to act as a cubing heuristic in C\&C, through two post-training stages. In the first stage, \emph{Supervised Fine-Tuning (SFT)} uses a dataset constructed by evaluating a teacher model on the training CNFs and recording its answers. In the second stage, \emph{Direct Preference Optimization (DPO)} consumes preference pairs generated by a \emph{Monte Carlo Tree Search (MCTS)} framework that systematically explores candidate splits (a mix of heuristic-driven and randomly sampled) and scores the resulting intermediate solver states against our reward function (\cref{eq:reward}). This yields a preference dataset covering a broader, more stateful distribution of cubing trajectories than CDCL solver trajectories alone, and we further augment each DPO decision with a teacher-generated reasoning trace, pairing solver-grounded decisions with neural reasoning.}

\textbf{Formalization.} \revision{We formalize the cubing problem as a Markov Decision Process $\mathcal{M} = (\mathcal{S}, \mathcal{A}, T, R, \gamma)$~\citep{russell1995modern}, where states $s \in \mathcal{S}$ are CNF formulas arising during C\&C, actions $\mathcal{A}(s) = \text{Vars}(\phi)$ are unassigned variables, and transitions are deterministic, i.e., the simplified subformulas are uniquely determined by the variable choice, so selecting variable $v$ in state $s$ (formula $\phi$) produces two child states $T(s, v) = (\text{simplify}(\phi \land v),\, \text{simplify}(\phi \land \neg v))$. The reward $R$ is defined in \cref{eq:reward} and satisfies boundedness ($R \in (0,1]$), monotonicity (increasing in unit propagations and eliminated clauses, decreasing in decisions and conflicts), and limiting behavior ($R=1$ when immediately solved, $R \to 0$ as difficulty grows). Full proofs are in \cref{app:theoretical}.}

\revision{This MDP formulation connects naturally to DPO~\citep{rafailov2023direct}, since the reward function (\cref{eq:reward}) scores individual splits, the cube scoring function (\cref{eq:score}) ranks candidate pairs by combining rewards from both branches, and DPO uses these rankings as preferences. Concretely, given prompt $x$ (a CNF formula) and two candidate responses $y_w, y_l$ (reasoning traces with cube decisions), the preference $y_w \succ y_l$ is determined by the scoring function, which favors balanced, high-quality splits. DPO then learns a policy $\pi_\theta$ that assigns higher likelihood to preferred decisions without explicit reward modeling.}

\subsection{Data Collection}\label{subsec:data-collection}

Training our pipeline requires benchmarks that serve as ground truth for reinforcement learning. \finalrev{We draw problems from the SAT Competition benchmark pool~\citep{heule2024proceedings}, the standard evaluation set for modern SAT solvers.} To obtain diverse and challenging formulas, we used the Global Benchmark Database~\citep{iser2024global}, which spans SAT Competitions from 2003 through 2025, and selected formulas with up to 600 variables and 1000 clauses, constrained by the context window.
We collected 94 SAT and 314 UNSAT formulas, sorted by normal CDCL solving time with a two-minute cutoff. For each category we reserved 42 easy formulas (solved by a CDCL solver in under 0.4 seconds) and 2 hard ones for SFT training, and filtered 30 easy formulas as the basis of our tree-based augmentation. For evaluation, we constructed a balanced test set of 50 SAT and 50 UNSAT formulas not used in training.
\finalrev{The test set spans ten instance families (\cref{tab:benchmark-families}) appropriate for Cube-and-Conquer's typical use case, namely hard random 3-SAT (\emph{sgen}, \emph{hgen}, Barthel), XOR / modular-arithmetic, proof-complexity-hard constraint graphs (Urquhart, Margulis), combinatorial puzzles (pigeonhole, Battleship), graph-encoded CSPs, cryptographic Hidden-Weighted-Bit circuits, and ISCAS-like circuits. It mixes easy and hard instances with an emphasis on hard ones, and under glucose v4.2 with a 24-hour timeout, $34\%$ of the benchmarks remain unsolved (\cref{fig:benchmark-difficulty-cdf}), which is the regime where Cube-and-Conquer is most useful.}
This dataset provides competition-grade instances but is too small to sustain reinforcement learning. Moreover, a model that performs reliably across the solving process needs SFT and DPO data that expose a wide variety of solver states, and not only the first cubing decision, grounded in trajectories that reflect how a solver actually branches. We therefore introduce an MCTS-based simulation of the solver-state space, which generates diverse trajectories of cubing decisions with well-defined decision rewards.

\subsection{SFT Dataset Creation}
\rebrev{To obtain supervised targets that reflect how a solver actually branches, we run the C\&C evaluation loop (\cref{alg:cube_and_conquer}) with the teacher model on the 88 reserved training CNFs (44 SAT and 44 UNSAT). Each time the teacher makes a well-formed cubing decision, we record the intermediate formula state it faced together with its reasoning trace and chosen split, so a single instance contributes many training pairs, one per decision along its solve trajectory, not just the root split. A decision counts as well-formed when it is formatted correctly and names a valid variable, regardless of whether that instance is eventually solved. The resulting 4332 (formula state, reasoning, and answer) pairs span a broad range of mid-solve states.}

\subsection{Decision Space Exploration}

To explore solver decision space, we employ Monte Carlo Tree Search (MCTS), a simulation-based algorithm widely used to navigate large discrete search spaces in domains such as planning and games. In our setting, a CNF formula can be viewed as the initial state of a search, with candidate variable assignments producing new states in the tree.
Here, the root node corresponds to the original CNF, and each child node represents the formula obtained by applying a splitting decision assignment to its parent.
This way, paths through the tree represent trajectories of cubing decisions.
We now describe the four stages of MCTS as instantiated in our formulation.

\textbf{Expansion.}
Given a parent node containing a CNF formula state, we expand children nodes by considering cube decision candidates. As it is intractable to evaluate every possible variable, we instead select up to six candidate variables, each expanded with both polarities, yielding at most twelve children per parent. Candidate variables are chosen from five existing symbolic heuristics implemented in Z3 \citep{de2008z3}. \zmarchcu is the default reward function of the March solver and serves as a well-established lookahead baseline, while \zternary is tuned to random 3-SAT and focuses on the structure of ternary clauses. \zheuleschur originates from the Schur Number 5 work of \citet{heule2018schur} and scores a literal by the sizes of its clauses and the occurrence counts of their other literals. \zheuleu keeps only the clause sizes, and \zunit extends \zheuleschur by also counting the unit clauses that the assignment would create. \revision{Their definitions and scoring functions are given in \cref{subsec:symbolic-heuristics}.} These heuristics often overlap to provide fewer than six unique proposals, so to maintain diversity we supplement them with randomly chosen variables until six candidates are obtained. Each resulting child node stores the reduced CNF formula after the variable assignment, together with relevant solver statistics.

\textbf{Selection.}
The selection phase chooses the child node that maximizes the standard UCT score~\citep{kocsis2006bandit,browne2012survey}, balancing exploitation (mean reward) and exploration (visit counts) with exploration constant $C = 2$. The selected child node then undergoes simulation.

\textbf{Simulation.}
The simulation phase, also known as rollout, calculates the value reward of the current leaf node, reflecting a signal from making a certain decision given a parent state. In our context, we need to define a reward that captures a splitting literal's efficacy in reducing and eventually solving the formula. We consider several symbolic statistics and present our custom reward:
\begin{equation}~\label{eq:reward}
   R_{node} = \frac{\log(1 + \varepsilon + u + e)}{\log(1 + \varepsilon + u + e + d + c)} \in (0,1]
\end{equation}
With a set of splitting variables, we create child nodes by applying these literal assignments to the parent CNF formula. This results in simplified CNF states, with implicitly propagated variable assignments and eliminated clauses. Intuitively, greater reductions lead us closer to solving the formula. In \cref{eq:reward}, $u$ denotes the number of unit propagations (variables implicitly propagated with no other satisfiable assignment) resulting from CNF reduction, while $e$ represents the number of eliminated clauses after literal assignment and subsequent unit propagation.
Given simplified CNF problems, we proceed to rollout using a CDCL solver with a timeout of 1 second. In doing so, we collate statistics measuring the difficulty of the CNF states. We specifically gather the number of splitting decisions $d$ and the number of conflict clauses $c$ generated by the solver, noting them as proxies for hardness. We consider these symbolic markers over pure solve time, which can be inconsistent and highly variant over runs. Our reward function satisfies nice properties such as \textit{boundedness} ($R \in (0, 1]$) and \textit{monotonicity} with respect to variables (formalisms in \cref{app:theoretical}). The $\varepsilon > 0$ term also ensures numerical stability by preventing $\log(1) = 0$ in degenerate cases.

\textbf{Backpropagation.}
After simulation, we propagate the current child's reward up to the root of the tree, updating the MCTS value of each node along the path to inform future selection decisions.

\subsection{DPO Dataset Creation}

Applying our MCTS formulation to 60 CNF problems, we build statefully-rich trees that simulate intermediary states of solver trajectories. We now examine how we transcribe these search trees as datasets for the training that requires contrastive pairs.
Specifically, DPO requires a training dataset with tuples of three elements $(prompt, chosen, rejected)$, with $prompt$ referring to a CNF problem and a call to produce a cubing decision.
We transcribe our tree by breath-first-search traversal. For each parent node, we collect the 12 children nodes and pair the positive-negative literals for 6 cubing decision candidates. The score of a cubing decision is computed as follows:
\begin{equation}~\label{eq:score}
score(l, \neg l) = (R_l \times R_{\neg l}) + (R_l + R_{\neg l}) \in (0,3]
\end{equation}
The multiplicative term rewards balanced cubes where both branches simplify well, while the additive term credits individually strong branches.
We take the maximal and minimal scoring pairs and add them to our dataset as \emph{chosen} and \emph{rejected} entries, respectively.

\textbf{Reasoning Distillation.}
\revision{While our symbolic preferences ground the model's outputs, training on decisions alone provides no reasoning signal for the model to generalize from. Moreover, without reasoning traces in the output, it would be impossible to analyze what heuristic strategies the model employs (\cref{sec:results}).}
To address this, we augment our DPO preference dataset with distilled reasoning. For each chosen and rejected decision, we query a teacher model to generate an explanatory trace, without revealing whether the decision was marked as chosen or rejected.
\revision{The teacher model reasoning is presented as a reasoning trace}, enclosed in \verb|<reasoning>| tags, prepended to the decision in \verb|<answer>| tags.

\begin{table}[t]
    \centering
    \footnotesize
    \setlength{\tabcolsep}{4pt}
    \caption{\revision{\Passat{k} scores} (unique benchmarks solved across $k$ attempts), per-run statistics, and satisfiability bias over 5 runs. \emph{Per-run avg}: SAT/50, UNSAT/50, $\mu$=mean total/100, $\sigma$=std dev. \emph{Satisfiability bias}: $\Delta$=SAT$-$UNSAT, $p$=paired t-test (\textbf{bold} if $p<0.05$), direction of significant bias. Full per-run breakdown in \cref{tab:extended-success-statistics}, with details in \cref{subsec:sat-bias}. \rebrev{This table shows the run-order-dependent \passat{k} scores, whereas \cref{subsec:run-order-independent-passk} provides the run-order-independent version.}}
    \label{tab:pass_at_k_scores}
    \vspace{0.5em}
    \begin{tabular}{@{}llccccc@{\hspace{18pt}}cccc@{\hspace{18pt}}rrl@{}}
        \toprule
        & & \multicolumn{5}{c}{\textbf{Pass@k (unique)}} & \multicolumn{4}{c}{\textbf{Per-run avg (5 runs)}} & \multicolumn{3}{c}{\textbf{Satisfiability bias}} \\
        \cmidrule(l{2pt}r{14pt}){3-7} \cmidrule(l{2pt}r{14pt}){8-11} \cmidrule(l{2pt}r{2pt}){12-14}
        & \textbf{Heuristic} & \textbf{@1} & \textbf{@2} & \textbf{@3} & \textbf{@4} & \textbf{@5} & \textbf{SAT} & \textbf{UNSAT} & $\boldsymbol{\mu}$ & $\boldsymbol{\sigma}$ & $\boldsymbol{\Delta}$ & $\boldsymbol{p}$ & \textbf{Bias} \\
        \midrule
        \multirow{5}{*}{\rotatebox[origin=c]{90}{\scriptsize\textit{LLM}}}
        & Qwen3-4B                  & 40 & 46 & 46 & 46 & 46 & 21.2 & 21.6 & 42.8 & 1.9 & $-0.4$ & \rebrev{.688} & none \\
        & GPT-OSS-120B              & 43 & 45 & 45 & 46 & 47 & 20.8 & 22.0 & 42.8 & 0.8 & $-1.2$ & \textbf{.033} & UNSAT \\
        & Claude-3.7-Sonnet         & 44 & 48 & 48 & 48 & 48 & 22.0 & 22.6 & 44.6 & 0.5 & $-0.6$ & .426 & none \\
        & Qwen3-32B                 & 46 & 48 & 49 & 49 & 49 & 21.0 & 22.0 & 43.0 & 2.4 & $-1.0$ & .089 & none \\
        & Claude-Sonnet-4           & 47 & 50 & 50 & 50 & 50 & 23.8 & 23.0 & 46.8 & 1.5 & $+0.8$ & .294 & none \\
        \midrule
        \multirow{3}{*}{\rotatebox[origin=c]{90}{\scriptsize\textit{Symbolic}}}
        & \zheuleschur              & 50 & 52 & 52 & 52 & 52 & 26.4 & 25.0 & 51.4 & 0.9 & $+1.4$ & \textbf{.025} & SAT \\
        & \zmarchcu                 & 51 & 52 & 52 & 52 & 52 & 26.6 & 25.0 & 51.6 & 0.5 & $+1.6$ & \textbf{.003} & SAT \\
        & \zunit                      & \textbf{52} & \textbf{53} & \textbf{53} & \textbf{53} & \textbf{53} & 26.4 & 25.2 & 51.6 & 0.5 & $+1.2$ & \textbf{.033} & SAT \\
        \midrule
        \multirow{3}{*}{\rotatebox[origin=c]{90}{\scriptsize\textit{Ours}}}
        & Qwen3-4B-DPO              & 42 & 45 & 45 & 46 & 47 & \rebrev{20.0} & 21.8 & \rebrev{41.8} & \rebrev{1.3} & \rebrev{$-1.8$} & \rebrev{\textbf{.009}} & UNSAT \\
        & Qwen3-4B-SFT              & \textbf{49} & 49 & 50 & 51 & 51 & 22.6 & 23.4 & 46.0 & 2.3 & $-0.8$ & .242 & none \\
        & \textbf{Qwen3-4B-SFT-DPO} & 46 & \textbf{51} & \textbf{51} & \textbf{53} & \textbf{53} & 23.8 & 23.6 & \textbf{47.4} & 0.9 & $+0.2$ & .778 & none \\
        \bottomrule
    \end{tabular}
\end{table}

\section{Results}\label{sec:results}

\textbf{Model.}
We use Qwen3-4B~\citep{yang2025qwen3} as our base model, with GPT-OSS-120b~\citep{gpt-oss-models} as our reasoning distillation teacher model.
While our main trained model is Qwen3-4B-SFT-DPO, we trained the based model with SFT only (Qwen3-4B-SFT) and with DPO only (Qwen3-4B-DPO) for ablation analysis.

\textbf{Datasets.}
\rebrev{Our SFT dataset contains 4332 samples, constructed as described in \cref{subsec:data-collection} using GPT-OSS-120B as the teacher model.} \revision{For the DPO preference dataset, we run 1000 MCTS iterations per formula over 60 training CNFs, we flattened the tree and collected the preference pairs. Then we prompted GPT-OSS-120B to \revision{generate reasoning traces} for each decision. To avoid out-of-memory issues during training, we filtered the resulting dataset keeping rows that could construct a prompt of maximum 8000 tokens. This way, the final dataset contained 9012 rows.}
 
\textbf{Hardware and Training.}
We perform distributed training on machines with 8 NVIDIA H100 80GB GPUs and 96 CPUs. We run the evaluation using one H100 or one NVIDIA A100 80G GPU, and 32 CPUs. For training, we used transformers~\citep{transformers-library}, transformers-reinforcement-learning~\citep{trl} and vLLM~\citep{vllm}, which we also use for inference. \finalrev{All LLMs use identical decoding configuration (same temperature and top-$p$ settings on the same vLLM server), so any across-model differences reflect training rather than decoding.}
Our training hyperparameters are in \cref{tab:training-information}

\begin{figure}[t]
    \centering
    \begin{subfigure}[t]{0.641\linewidth}
        \centering
        \includegraphics[width=\linewidth]{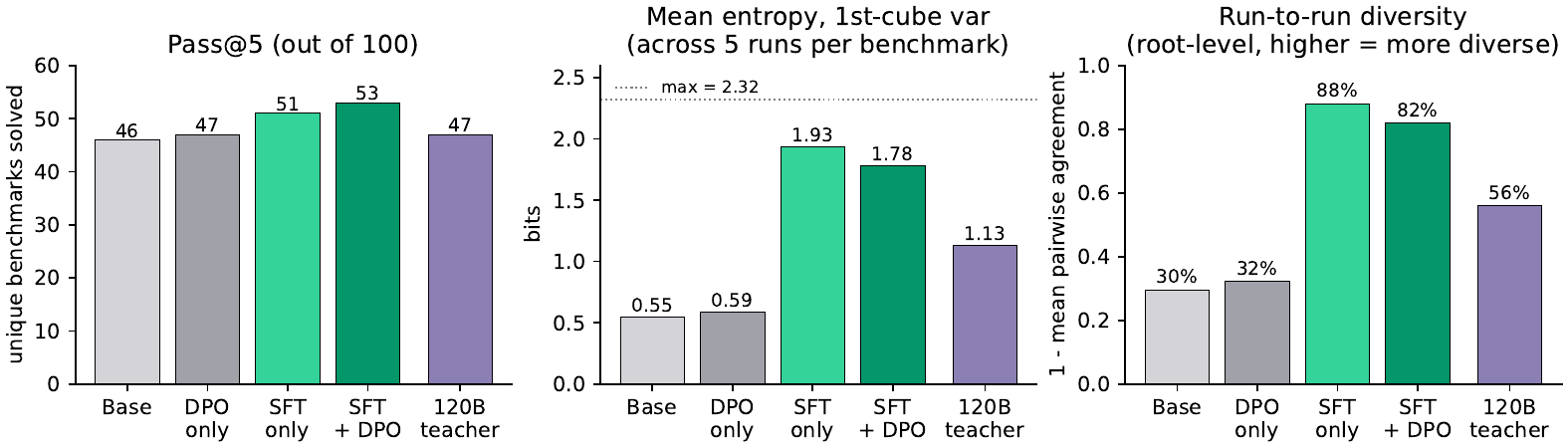}
        \caption{{Stage ablation of the 4B model over \passat{5}, first-cube entropy, and first-cube diversity, with the 120B teacher (rightmost) shown for reference.}}
        \label{fig:training-stage-ablation}
    \end{subfigure}
    \hfill
    \begin{subfigure}[t]{0.338\linewidth}
        \centering
        \includegraphics[width=\linewidth]{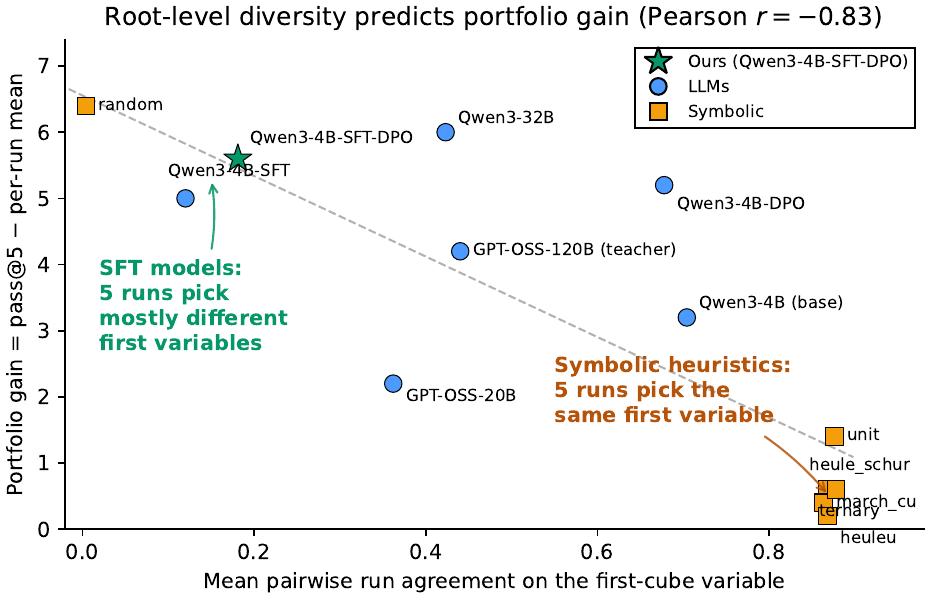}
        \caption{{Portfolio gain vs first-cube agreement, one dot per heuristic.}}
        \label{fig:agreement-vs-gain-main}
    \end{subfigure}
    \caption{SFT unlocks root-level decision diversity (a), and that diversity predicts portfolio gain across heuristics (b). \emph{(a)} reports \passat{5}, mean first-cube Shannon entropy across 5 runs per benchmark (max $\log_2 5 \approx 2.32$ bits), and first-cube diversity ($1 -$ mean pairwise run agreement) for the four training stages of the 4B model plus the 120B teacher as a reference. \emph{(b)} plots \emph{portfolio gain} ($\passat{5} -$ per-run mean) against first-cube agreement for all 13 heuristics in \cref{tab:pass_at_k_scores}, with Pearson \rebrev{$r = -0.83$}. Cross-heuristic context in \cref{subsec:diversity}; annotated version of~(b) in \cref{fig:agreement-vs-gain}.}
    \label{fig:mechanism}
\end{figure}

\textbf{Evaluation Setup.}
We run the evaluation on 50 SAT and 50 UNSAT benchmarks from the test set, integrating each heuristic in the \texttt{Choose-Split} method of \cref{alg:cube_and_conquer}. Baselines include symbolic heuristics implemented in Z3~\citep{de2008z3}(\zunit, \zternary, \zmarchcu, \zheuleu, \zheuleschur), implemented in python (random) and LLMs (Qwen3-4B, Qwen3-32B, GPT-OSS-20B, GPT-OSS-120B, Claude-Sonnet-3.7, Claude-Sonnet-4). \revision{If LLMs could not make a valid decision for a child node (holding a sub-formula) after 10 attempts, C\&C aborted that node and continued on others.} The end-to-end C\&C timeout for each benchmark is 30 minutes, while the intermediate CDCL solving timeout is 5 seconds. For the CDCL process we used the glucose v4.2 solver~\citep{audemard2018glucose}, based on MiniSat~\citep{een2005minisat}. The evaluation was parallelized over 4 processes (that used the same vLLM server), evenly split between SAT and UNSAT problems. For each heuristic, we run five iterations of the 100 CNF benchmarks.

\textbf{Solving Performance.}
\Cref{tab:pass_at_k_scores} presents the \passat{k} scores of different heuristics\rebrev{, which are run-order-dependent as they count the first $k$ runs in execution order}. \revision{This score calculates the number of \emph{unique} benchmarks solved across up to $k$ independent attempts. Using five attempts, our base model (Qwen3-4B) achieves an initial 46 \passat{5} score, supervised fine-tuning (Qwen3-4B-SFT) raises this to 51, and adding DPO (Qwen3-4B-SFT-DPO) reaches 53. \rebrev{Without SFT, the DPO-only variant (Qwen3-4B-DPO) reaches a \passat{5} of only 47, since DPO reweights the reference policy toward preferred decisions rather than instilling capability the base lacks.} Comparatively, \zunit, the best performing symbolic heuristic, also achieves a 53 \passat{5} score, while Claude-Sonnet-4 reaches only 50.} \rebrev{On a single attempt, the gap between our trained model (Qwen3-4B-SFT-DPO) and \zunit is 6 benchmarks (46 against 52). \Cref{subsec:run-order-independent-passk} gives a more robust, run-order-independent comparison that averages over all unordered choices of $k$ runs narrowing this gap to 4.2 benchmarks (47.4 against 51.6).}

\textbf{Budget Robustness.}
\rebrev{The comparison with \zunit is also robust to the per-attempt time budget. Because the timeout only stops a run and does not influence its decisions, truncating each run post-hoc is exact, and recomputing coverage at tighter budgets shows that our model stays at or ahead of \zunit on \passat{5} at every budget from 5 to 30 minutes (\cref{tab:budget-sweep}). The margin is largest at the tightest budget, since the remaining advantage of \zunit lies in a long tail of slow solves, and on the per-run mean \zunit leads only from 15 minutes upward.}

\textbf{Per-Run Variation.}
\revision{\Cref{tab:extended-success-statistics} reports the raw per-run scores across all five attempts. Symbolic heuristics solve nearly the same benchmarks in every run (\zunit averages 51.6 of 100 per run), whereas our model solves a shifting set (Qwen3-4B-SFT-DPO averages 47.4 per run). This shift is productive, since the gap between the per-run average (47.4) and \passat{5} (53) shows that different runs solve \emph{different} benchmarks.} \rebrev{On top of SFT, SFT-DPO lowers the run-to-run variation in the total count ($\sigma$ of 0.9 against 2.3 for SFT alone) while retaining benchmark-level diversity, but this is an observed reduction across five sampled runs of a single trained model and not an established reliability advantage. The $+2$ gain in \passat{5} (51 to 53) is too small on 100 benchmarks to certify on its own, and the per-run mean, which rises from 46.0 (SFT) to 47.4 (SFT-DPO), supports it more robustly.}

\textbf{Decision Diversity.}
\revision{This per-run diversity is not a generic LLM property (\cref{fig:mechanism}). The untrained Qwen3-4B base and the DPO-only ablation are near-deterministic at the root, with first-cube entropy of 0.55--0.59 bits and 68--70\% run agreement, which is much closer to the deterministic symbolic heuristics (0.22--0.25 bits and 86--88\%) than to the SFT models. SFT moves the model into a qualitatively different regime, in which Qwen3-4B-SFT reaches 1.93 bits of entropy and 12\% run agreement and Qwen3-4B-SFT-DPO retains 1.78 bits and 18\%.} \rebrev{SFT, rather than DPO, is therefore the stage that raises root-level diversity, and this diversity is associated with the gap between \passat{5} and the per-run mean without being shown to cause it. DPO on top of SFT trades a little diversity for two additional \passat{5} benchmarks. Across all 13 heuristics, first-cube agreement and portfolio gain are negatively correlated (Pearson $r = -0.83$, or $-0.85$ with AlphaMapleSAT as in \cref{fig:agreement-vs-gain}), an association that is not specific to our model. Because first-cube agreement only lower-bounds true trajectory diversity for the LLMs, near-root-deterministic models such as the DPO-only ablation still gain from deeper-level stochasticity and timeout-marginal instances, as \cref{tab:flip-benchmarks} shows, which places that point above the regression line.} The 4B student also exceeds its 120B distillation teacher on all three axes, with 53 against 47 in \passat{5}, 1.78 against 1.13 bits in entropy, and 18\% against 44\% in run agreement, as \cref{tab:first-cube-comparison} shows in full, and \cref{subsec:diversity} gives the cross-heuristic comparison and offers a hypothesis for why.

\textbf{Complementary Coverage.}
\Cref{fig:success-heatmap-extended} in the appendix shows, for each benchmark (indexed 0--49 for SAT and 50--99 for UNSAT) and each of the five attempts per heuristic, whether the benchmark was solved (green) or not solved within the 30-minute timeout (yellow).
\finalrev{In \passat{5}, Qwen3-4B-SFT-DPO solves two benchmarks (24, 73) that \zunit does not, and \zunit solves two benchmarks (30, 77) that our model does not, which illustrates the complementary coverage that motivates a portfolio combination.} \rebrev{A two-attempt portfolio quantifies this complementarity (\cref{subsec:mixed-portfolio}), although it gains only marginally from the learned policy over a random second attempt.}

\textbf{SAT and UNSAT Performance.}
\revision{\Cref{tab:pass_at_k_scores} also reveals SAT/UNSAT performance differences. Paired t-tests (\cref{subsec:sat-vs-unsat}) show that most heuristics are balanced, with five exceptions. The symbolic \zunit, \zheuleschur, and \zmarchcu heuristics all lean SAT ($p=0.033$, $0.025$, $0.003$ respectively), consistent with the SAT split's concentration on hard random 3-SAT instances (\cref{subsec:benchmark-families,subsec:sat-vs-unsat}), a class where lookahead heuristics like these typically outperform CDCL alone. GPT-OSS-120B leans UNSAT (20.8 vs 22.0, $p=0.033$), and Qwen3-4B-DPO (DPO only, no SFT) has a UNSAT bias (\rebrev{20.0} vs 21.8, \rebrev{$p=0.009$}). Notably, our main model Qwen3-4B-SFT-DPO is well-balanced (SAT=23.8, UNSAT=23.6, $p=0.778$), suggesting that SFT provides a balanced foundation that DPO preserves.}

\begin{figure*}[t]
    \tiny
    \centering
    \begin{subfigure}{0.365\textwidth}
        \tiny
        \centering
        \includegraphics[width=\linewidth]{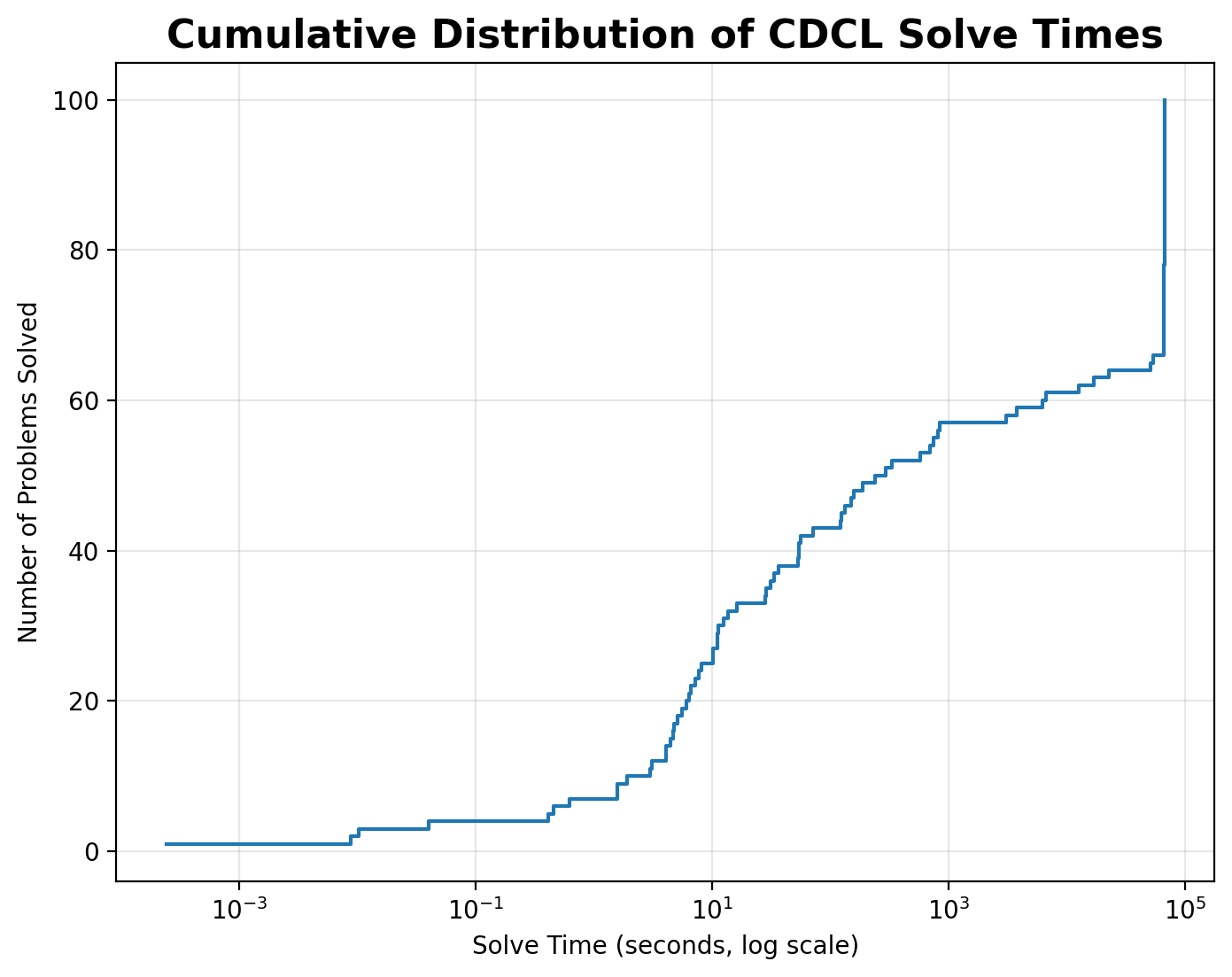}
        \caption{Cumulative distribution of test benchmark CDCL solve times (timeout =24h). 34 problems remained unsolved.}
        \label{fig:benchmark-difficulty-cdf}
    \end{subfigure}
    \hfill
    \begin{subfigure}{0.60\textwidth}
        \tiny
        \centering
        \includegraphics[trim=0 0 0 10em, clip, width=\linewidth]{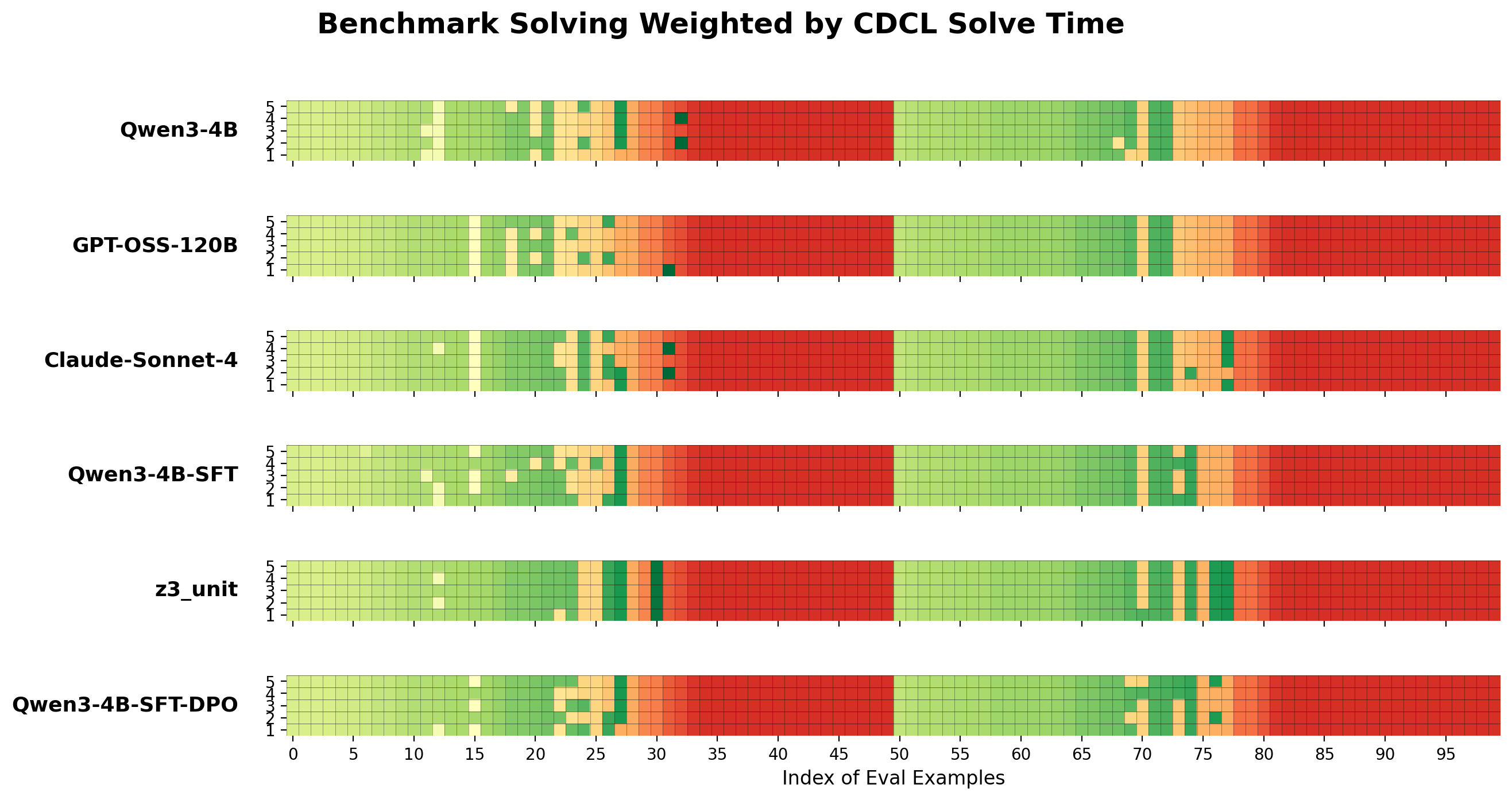}
        \vspace{-0.2cm}
        \caption{Solved benchmarks weighted and ordered by CDCL solve time [0-49 sat, 50-99 unsat]. Green implies solved benchmark, red is unsolved ($<$30min). Deeper color denotes longer CDCL solve time. }
        \label{fig:cnc-time-heatmap}
    \end{subfigure}
    \caption{CDCL solve time (without cubing) acts as a proxy for problem difficulty.}
\end{figure*}

\begin{figure*}[t]
    \tiny
    \begin{subfigure}{0.37\linewidth}
        \tiny
        \centering
        \includegraphics[width=\textwidth]{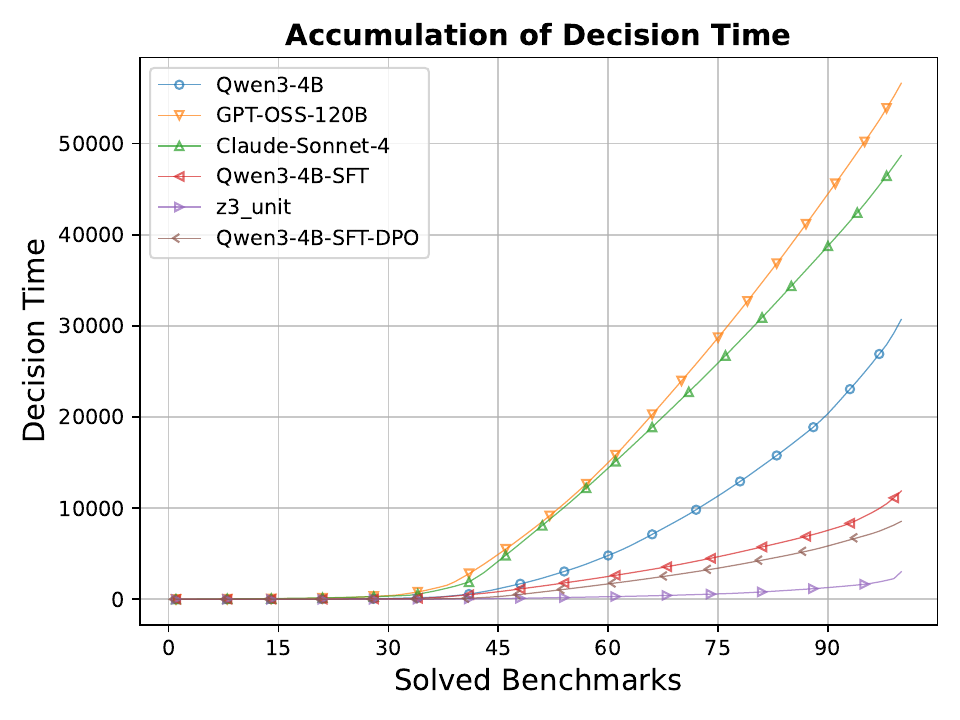}
        \caption{Accumulated decision time (sec) per heuristic\rebrev{, over each heuristic's best SAT and best UNSAT attempt in \cref{tab:extended-success-statistics}}.}
        \label{fig:accumulated-decision-time}
    \end{subfigure}
    \hfill
    \begin{subfigure}{0.6\linewidth}
        \tiny
        \centering
        \includegraphics[trim=0 0 0 10em, clip, width=\textwidth]{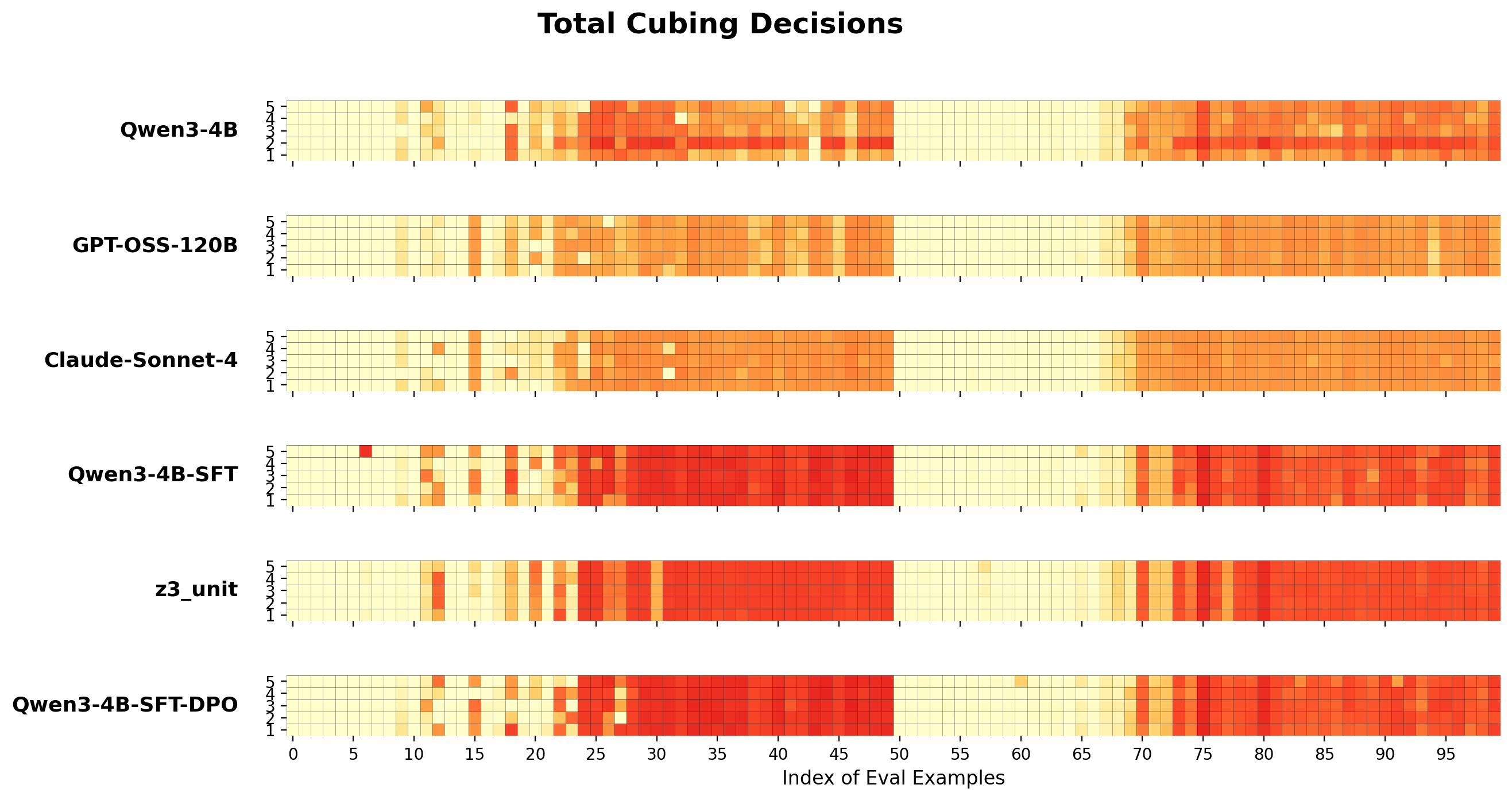}
        \caption{Total cubing decisions made per benchmark (deeper color denotes more decisions). Symbolic heuristics and smaller models make more decisions in allotted time.}
        \label{fig:total-decisions-heatmap}
    \end{subfigure}
    \caption{Larger models have larger response time, leading them to fewer total decisions than the symbolic heuristics.}
\end{figure*}

\textbf{Benchmark Difficulty.}
To better understand the difficulty distribution of the chosen benchmarks, we measure the time to solve each benchmark with a CDCL solver as a proxy of hardness.
\Cref{fig:benchmark-difficulty-cdf} displays the cumulative solve time distribution of the test benchmarks, running the glucose v4.2 solver with a timeout of 24h.
We see that roughly $58\%$ of our benchmarks could be solved in 1000 sec (16.6 min), while $34\%$ of them could not be solved in 24 hours (top right vertical line). \revision{To showcase this difficulty in our experiments, \Cref{fig:cnc-time-heatmap} presents the time each C\&C run spent for SAT solving (conquering phase). Here, the cell color (green, red) indicates whether the benchmark was solved, while the color hue indicates benchmark difficulty. Deeper green indicates successful but time consuming attempts, while deeper red indicates that C\&C timed out. We observe that all heuristics performed similarly on easier (or harder) problems, with a few of the medium-challenging problems acting as the differentiators.}

\textbf{Decision Time.}
\finalrev{Decision time bounds the number of splits a heuristic can make within the 30-minute timeout. \Cref{fig:accumulated-decision-time} reports the cumulative time each heuristic spends on decisions, taken over its best SAT and best UNSAT attempts, and separates three regimes. The large LLMs (Claude-Sonnet-4, GPT-OSS-120B) incur the highest decision cost, the symbolic \zunit heuristic the lowest, and our fine-tuned models fall in between together with the base Qwen3-4B, which motivates the use of a small LLM for end-to-end C\&C. Qwen3-4B-SFT-DPO decides faster than the base model despite an identical parameter count, which the inference logs attribute to fewer output-formatting retries. \rebrev{This conclusion is unchanged when decision time is aggregated over all five runs rather than a single attempt, shown in \cref{subsec:decision-time-accounting}. Deciding then accounts for 9.4\% of Qwen3-4B-SFT-DPO's end-to-end C\&C time, compared with 39.5\% for the base model and 3.5\% for \zunit, so the decision overhead of our model remains a small fraction of the budget.}}

\textbf{Decision Count.}
\finalrev{\Cref{fig:total-decisions-heatmap} shows the same three tiers in decision count. The slowest LLMs make the fewest decisions, which together with their lower \passat{5} scores supports the claim that response time hinders exploration. Decision count is nevertheless not a measure of solution quality, since Claude-Sonnet-4 solves more problems than faster heuristics such as Qwen3-4B.} \rebrev{It also depends on how quickly each run's search path terminates, so we interpret it as a throughput signal and not as an isolated measure of speed.}

\textbf{Learned Heuristic Analysis.}
A post-hoc LLM-as-judge classified 388 reasoning traces into established SAT heuristic categories (\cref{tab:heuristic-top7}). \emph{Variable Frequency Analysis} (H1) is dominant at 73.5\%, but \cref{subsec:freq-rank} shows that every LLM we measured, including the untrained base, shares this preference, so we read it as a calibration observation rather than a training outcome. To validate the judge, two human experts labeled a random sample of 50 traces. The judge agrees with each of them moderately (Cohen's $\kappa$ of 0.45 and 0.48), which is within the range of disagreement between the humans themselves ($\kappa$ of 0.37), and \cref{subsec:inter-rater} gives the full methodology. \rebrev{The moderate agreement reflects traces that often invoke several heuristics at once, and none of our behavioral findings depend on these labels, since the diversity mechanism of \cref{subsec:diversity} uses only realized first-cube variables from the solver logs.}

\begin{table}[t]
    \centering
    \footnotesize
    \caption{\finalrev{\emph{Left block}: top-7 primary heuristics on the full 388-trace classification (\emph{Cnt}: trace count, \emph{\%}: share of 388, \emph{Pri}: traces where this was the primary label). \emph{Right block}: 50-trace inter-rater validation, reporting how Human A, Human B, and Claude-Sonnet-4.5 assigned primary labels, and \emph{Claude $\kappa$ vs A/B} are per-heuristic Cohen's $\kappa$ against each human. Summary row: overall pairwise $\kappa$ and Fleiss' $\kappa$ across all three raters. Full 11-row distribution in \cref{tab:heuristic-classification}.}}
    \label{tab:heuristic-top7}
    \vspace{0.4em}
    \begin{tabular}{@{}lllrrrrrrrr@{}}
        \toprule
        & & & \multicolumn{3}{c}{\textbf{Full (n=388)}} & \multicolumn{5}{c}{\textbf{Validation (n=50)}} \\
        \cmidrule(lr){4-6} \cmidrule(lr){7-11}
        \textbf{ID} & \textbf{Heuristic} & \textbf{Literature} & \textbf{Cnt} & \textbf{\%} & \textbf{Pri} & \textbf{A} & \textbf{B} & \textbf{Cl} & \textbf{$\kappa_{A,Cl}$} & \textbf{$\kappa_{B,Cl}$} \\
        \midrule
        H1 & Variable Frequency       & \cite{moskewicz2001chaff}     & 285 & 73.5 & 192 & 16 & 21 & 24 & 0.43 & 0.56 \\
        H2 & Balanced Splitting       & \cite{heule2011cube}          & 170 & 43.8 &  60 &  7 &  4 &  6 & 0.56 & 0.34 \\
        H6 & Complexity Reduction     & \cite{heule2011cube}          & 169 & 43.6 &  37 &  6 &  3 &  9 & 0.30 & 0.45 \\
        H7 & Polarity Balance         & \cite{silva1996grasp}         &  89 & 22.9 &  59 & 10 & 12 &  9 & 0.81 & 0.58 \\
        H5 & Structural Pattern       & \cite{ansotegui2019community} &  33 &  8.5 &   9 &  4 &  3 &  1 & $-$0.03 & 0.48 \\
        H4 & Clause-Size Weighting    & \cite{jeroslow1990solving}    &  32 &  8.2 &   3 &  0 &  0 &  0 & --- & --- \\
        H3 & Clause Elimination       & \cite{davis1962machine}       &  30 &  7.7 &   5 &  2 &  3 &  0 & 0.00 & 0.00 \\
        \midrule
        \multicolumn{11}{@{}l}{\textbf{Overall} \quad $\kappa_{A,B} = 0.37$ \quad $\kappa_{A,Cl} = 0.45$ \quad $\kappa_{B,Cl} = 0.48$ \quad\textbf{Fleiss' $\kappa$ (A + B + Cl) = 0.43}} \\
        \bottomrule
    \end{tabular}
\end{table}

\section{Limitations}
\finalrev{We note five limitations. (1) Context-window constraints restrict us to formulas with up to 600 variables and 1000 clauses, \rebrev{a ceiling well below the large, hard instances where C\&C is most valuable, so that regime remains untested}. Sliding-window attention, formula chunking, and hierarchical variable selection are natural mitigations but remain unexplored. (2) \rebrev{The training data compounds this gap, since it consists mostly of easy instances because MCTS data generation needs quick solver rollouts at every node, and representing the harder test distribution more fully during training remains a promising direction.} (3) At inference, each decision is slower than for CPU-based symbolic heuristics (\cref{fig:accumulated-decision-time}), though \rebrev{deciding takes under a tenth of our model's end-to-end C\&C time (\cref{tab:decision-time-accounting}) and} our 30-minute-per-benchmark evaluation shows this does not prevent matching the best symbolic heuristic on \passat{5}. (4) Our mechanism analysis (\cref{fig:mechanism}) examines only the first-cube decision, so whether the diversity associated with SFT persists at depth is open, \rebrev{and the attribution of this diversity to SFT rather than DPO is tested only on the one model family we trained through both stages. The two-benchmark gain of DPO over SFT is also observed on five sampled runs of a single trained model over 100 benchmarks, which is too few to certify it on its own.} (5) Our findings about SFT-associated diversity may not transfer to instance families with very different clause structure, such as random 3-SAT at the phase transition.}

\section{Conclusion}

We presented the first successful approach for training transformers to learn effective cubing heuristics for Cube-and-Conquer SAT solving. Our neuro-symbolic post-training framework enables a 4B-parameter model to achieve \passat{5} of 53, matching the best symbolic heuristic and surpassing frontier LLMs, \rebrev{when each is given five attempts per benchmark}. \rebrev{Given a single attempt, the best heuristic remains the symbolic one.} Ablations show that SFT improves the base model from 46 to 51 and that DPO adds two further benchmarks, reaching 53, which indicates that transformers can acquire effective cubing policies with appropriate supervision. \finalrev{An ablation on realized first-cube decisions further associates root-level decision diversity with SFT rather than DPO.}
\finalrev{Unlike the Z3 symbolic heuristics that show a significant SAT bias, \revision{our model is SAT/UNSAT-balanced and solves complementary benchmarks}, making it a candidate for portfolio-based solvers.} \rebrev{A further direction is to scale the policy to larger formulas by partitioning them into context-window-sized chunks and selecting each cube by majority vote across chunks.} \finalrev{The same framework could apply to other NP-hard decision problems whose solvers rely on learned branching heuristics, such as QBF~\citep{zhang2002conflict,biere2009handbook} and SMT~\citep{de2008z3}.}

\begin{ack}
This work  was partially supported by the National Science Foundation under Grant
Numbers CCF-2318974, CCF-2219995, CNS-2245344, and by the Defense Advanced Research Projects Agency (DARPA) under Agreement No. HR00112590130.
\end{ack}

\bibliographystyle{plainnat}
\bibliography{references}

\newpage
\appendix
\section{Appendix}

\subsection{Symbolic Heuristic Definitions}~\label{subsec:symbolic-heuristics}

We compare our learned heuristics against several symbolic lookahead heuristics implemented in Z3~\citep{de2008z3}. These heuristics determine which literal to select for cubing by computing reward scores based on different structural properties of the formula.

\subsubsection{Ternary}
The \zternary heuristic is a reward function optimized for random 3-SAT instances, originally used by Heule and Knuth in the March solver. It focuses on the structure of ternary (3-literal) clauses to guide splitting decisions.

\subsubsection{Heule-Schur}
The \zheuleschur heuristic is based on the ``Schur Number 5'' work by~\citep{heule2018schur}, a combinatorial / Ramsey-theory construction problem. The score of a literal $\ell$ is computed as:
\begin{equation}
    \text{score}(\ell) = \sum_{C \in \mathcal{C} : \ell \in C} 2^{-|C|+1} \cdot \frac{\sum_{\ell' \in C, \ell' \neq \ell} \text{occs}(\neg \ell')}{|C|}
\end{equation}
where $\mathcal{C}$ is the set of all clauses, $|C|$ is the size of clause $C$, and $\text{occs}(\ell)$ denotes the number of clauses containing literal $\ell$.

\subsubsection{Heule-Unit}
The \zheuleu heuristic uses a simplified scoring function:
\begin{equation}
    \text{score}(\ell) = \sum_{C \in \mathcal{C} : \ell \in C} 2^{-|C|+1}
\end{equation}
This variant focuses purely on clause sizes without considering literal occurrences and similar to \zheuleschur it descends from the Schur-number line of combinatorial benchmarks.

\subsubsection{Unit}
The \zunit heuristic extends \zheuleschur by additionally counting the number of unit clauses that would be created by the assignment. This provides additional weight to decisions that lead to immediate simplifications through unit propagation. It inherits \zheuleschur's combinatorial lineage.

\subsubsection{March-CU}
The \zmarchcu heuristic is the default reward function used in a version of the March solver, which has historically been applied to both hard combinatorial problems (Schur-5, Pythagorean triples, van der Waerden) and hard random 3-SAT. It represents a well-established lookahead baseline.

\textbf{Mix-Diff Variants.} Each reward function comes with its own variant of ``mix\_diff'', which is the function for combining reward metrics for the positive and negative variants of a literal. This ensures balanced consideration of both polarities when making splitting decisions.

\subsection{Theoretical Analysis of the Learning Framework}~\label{app:theoretical}

We provide a formal mathematical analysis of our neuro-symbolic training framework, including the MDP formulation, reward function properties, and training objectives.

\subsubsection{MDP Formulation for Cubing}

We formalize the cubing problem as a Markov Decision Process (MDP) $\mathcal{M} = (\mathcal{S}, \mathcal{A}, T, R, \gamma)$~\citep{russell1995modern}, with the components defined below.

\textbf{State Space $\mathcal{S}$:} The set of all CNF formulas $\phi$ that can arise during cube-and-conquer. Each state $s \in \mathcal{S}$ represents a CNF formula after zero or more splitting decisions have been applied.

\textbf{Action Space $\mathcal{A}(s)$:} For a state $s$ representing formula $\phi$, the action space is $\mathcal{A}(s) = \text{Vars}(\phi)$, the set of unassigned variables in $\phi$. Selecting action $v$ corresponds to choosing variable $v$ as the splitting variable.

\textbf{Transition Function $T$:} The transition is deterministic. Given state $s$ (formula $\phi$) and action $v$, the transition produces two child states:
\begin{equation}
    T(s, v) = \left( \text{simplify}(\phi \land v), \text{simplify}(\phi \land \neg v) \right)
\end{equation}
where $\text{simplify}(\cdot)$ applies unit propagation and clause elimination.

\textbf{Reward Function $R$:} As defined in \cref{eq:reward}, the reward captures the quality of a splitting decision.

\textbf{Discount Factor $\gamma$:} In our formulation, we use $\gamma = 1$ (undiscounted) since we care about total solving efficiency.

The optimal policy $\pi^*$ satisfies a \rebrev{Bellman-style recursion}:
\begin{equation}~\label{eq:bellman}
    V^*(s) = \max_{v \in \mathcal{A}(s)} \left[ R(s, v) + \frac{1}{2} \left( V^*(s_v^+) + V^*(s_v^-) \right) \right]
\end{equation}
where $s_v^+ = \text{simplify}(\phi \land v)$ and $s_v^- = \text{simplify}(\phi \land \neg v)$ are the two child states. \rebrev{The factor of $\frac{1}{2}$ averages the two child values, a modeling choice that keeps $V^*$ on the same bounded scale as $R$ and favors balanced splits rather than a cost derived from requiring both branches to be solved.}

\subsubsection{Reward Function Analysis}

We analyze the mathematical properties of our reward function from \cref{eq:reward}:
\begin{equation}
    R_{\text{node}} = \frac{\log(1 + \varepsilon + u + e)}{\log(1 + \varepsilon + u + e + d + c)}
\end{equation}
where $\varepsilon > 0$ is a small constant ensuring numerical stability by preventing $\log(1) = 0$ in degenerate cases where all metrics are zero.

\begin{proposition}[Boundedness]
For all non-negative values $u, e, d, c \geq 0$ and $\varepsilon > 0$, we have $R_{\text{node}} \in (0, 1]$.
\end{proposition}

\begin{proof}
Let $N = 1 + \varepsilon + u + e$ (numerator argument) and $D = 1 + \varepsilon + u + e + d + c$ (denominator argument).

\textit{Upper bound:} Since $d, c \geq 0$, we have $D \geq N > 1$. Therefore $\log(D) \geq \log(N) > 0$, which implies $R = \frac{\log(N)}{\log(D)} \leq 1$ with equality when $d = c = 0$.

\textit{Lower bound:} Since $N \geq 1 + \varepsilon > 1$ and the logarithm is positive for arguments greater than 1, both numerator and denominator are positive. Thus $R > 0$.
\end{proof}

\begin{proposition}[Monotonicity]
The reward function $R$ is (1) monotonically increasing in $u$ and $e$ (unit propagations and eliminated clauses), and (2) monotonically decreasing in $d$ and $c$ (decisions and conflicts).
\end{proposition}

\begin{proof}
Let $N = 1 + \varepsilon + u + e$ and $D = N + d + c$. Then $R = \frac{\log N}{\log D}$.

\textit{Increasing in $u$ (and similarly $e$):}
$\frac{\partial R}{\partial u} = \frac{D \cdot \log D - N \cdot \log N}{N \cdot D \cdot \log^2D}$.

We have $D \geq N > 1$, so $\log^2D > 0$, thus we need to show $D \cdot \log D - N \cdot \log N \geq 0 \implies D \cdot \log D \geq N \cdot \log N$.

So, $D \geq N > 1 \implies \log D \geq \log N > 0$, thus $D \cdot \log D \geq N \cdot \log N$, thus $D \cdot \log D \>\geq N \cdot \log N$.

\textit{Decreasing in $d$ (and similarly $c$):}
$\frac{\partial R}{\partial d} = -\frac{\log N}{D \cdot \log^2D} < 0$
since $\log N > 0$, $D \geq N > 1$, and $\log D \geq \log N > 0$.
\end{proof}

\begin{proposition}[Limiting Behavior]
(1) When $d = c = 0$ (immediately solved formula), $R = 1$. (2) As $d + c \to \infty$ (very hard formula), $R \to 0$.
\end{proposition}

\begin{proof}
\textit{Case 1:} When $d = c = 0$, we have $D = N$, so $R = \frac{\log N}{\log N} = 1$.
\textit{Case 2:} As $d + c \to \infty$, $D \to \infty$ while $N$ remains fixed. Thus $\lim_{D \to \infty} \frac{\log N}{\log D} = 0$.
\end{proof}

\subsubsection{DPO Training Objective}

Direct Preference Optimization (DPO)~\citep{rafailov2023direct} learns from preference pairs without explicit reward modeling. Given a dataset $\mathcal{D} = \{(x_i, y_w^i, y_l^i)\}$ where $y_w$ is preferred over $y_l$ for prompt $x$, the DPO loss is:
\begin{equation}~\label{eq:dpo-loss}
    \mathcal{L}_{\text{DPO}}(\pi_\theta; \pi_{\text{ref}}) = -\mathbb{E}_{(x, y_w, y_l) \sim \mathcal{D}} \left[ \log \sigma \left( \beta \cdot h_\theta(x, y_w, y_l) \right) \right]
\end{equation}
where $\sigma$ is the sigmoid function, $\beta$ is a temperature parameter, and:
\begin{equation}
    h_\theta(x, y_w, y_l) = \log \frac{\pi_\theta(y_w | x)}{\pi_{\text{ref}}(y_w | x)} - \log \frac{\pi_\theta(y_l | x)}{\pi_{\text{ref}}(y_l | x)}
\end{equation}

The key insight of DPO is that it implicitly optimizes an equivalent objective to RLHF under the Bradley-Terry preference model. Specifically, DPO learns a policy $\pi_\theta$ such that the implicit reward $r^*(x, y) = \beta \log \frac{\pi_\theta(y | x)}{\pi_{\text{ref}}(y | x)}$ satisfies $r^*(x, y_w) > r^*(x, y_l)$ for preferred responses.

\paragraph{Connection to Cubing.} In our setting, $x$ is the CNF formula (prompt), $y$ is the reasoning trace and cube decision, and preference $y_w \succ y_l$ is determined by our cube scoring function (\cref{eq:score}).

\subsubsection{Cube Scoring Function Analysis}

Our cube scoring function from \cref{eq:score} combines rewards from both branches:
\begin{equation}
    \text{score}(l, \neg l) = \underbrace{R_l \cdot R_{\neg l}}_{\text{balance term}} + \underbrace{R_l + R_{\neg l}}_{\text{progress term}}
\end{equation}

\begin{proposition}[Symmetry]
The scoring function is symmetric, since $\text{score}(l, \neg l) = \text{score}(\neg l, l)$.
\end{proposition}
\begin{proof}
Both the product $R_l \cdot R_{\neg l}$ and sum $R_l + R_{\neg l}$ are symmetric in their arguments.
\end{proof}

\begin{proposition}[Balance Incentive]
For fixed total progress $R_l + R_{\neg l} = S$, the score is maximized when $R_l = R_{\neg l} = S/2$ (balanced branches).
\end{proposition}
\begin{proof}
Given constraint $R_l + R_{\neg l} = S$, let $R_l = t$ and $R_{\neg l} = S - t$. The score becomes $f(t) = t(S-t) + S = -t^2 + St + S$.
Taking the derivative gives $f'(t) = -2t + S = 0 \implies t = S/2$.
The second derivative $f''(t) = -2 < 0$ confirms this is a maximum.
\end{proof}

\begin{proposition}[Score Range]
For rewards $R_l, R_{\neg l} \in (0, 1]$, the score satisfies $\text{score}(l, \neg l) \in (0, 3]$ with maximum achieved when $R_l = R_{\neg l} = 1$.
\end{proposition}
\begin{proof}
\textit{Upper bound:} When $R_l = R_{\neg l} = 1$: $\text{score} = 1 \cdot 1 + 1 + 1 = 3$.
\textit{Lower bound:} As $R_l, R_{\neg l} \to 0^+$, both terms approach 0, so $\text{score} \to 0^+$.
\end{proof}

\subsubsection{Comparison with Alternative Reward Designs}

We considered several alternative reward formulations, listed below.

\paragraph{Additive Reward.} $R_{\text{add}} = \alpha(u + e) - \beta(d + c)$.
This linear formulation lacks boundedness and can produce unbounded negative rewards, destabilizing training.

\paragraph{Multiplicative Reward.} $R_{\text{mult}} = \frac{u + e + 1}{d + c + 1}$.
While bounded below by 0, this can grow unboundedly large (depending on the formula's number of clauses) when $d = c = 0$ but $u, e$ are large.

\paragraph{Time-Based Reward.} $R_{\text{time}} = \frac{1}{1 + t_{\text{solve}}}$.
Directly using solve time $t_{\text{solve}}$ introduces high variance due to system noise, caching effects, and non-determinism in parallel solvers.

Our log-ratio formulation (\cref{eq:reward}) combines the benefits of boundedness, stability, and monotonicity while avoiding the drawbacks of these alternatives.

\subsection{Training Hyperparameters}

\revision{\Cref{tab:training-information} reports the full set of training hyperparameters used to produce the three model variants (Qwen3-4B-DPO, Qwen3-4B-SFT, Qwen3-4B-SFT-DPO) discussed in \cref{sec:results}. The end-to-end neuro-symbolic post-training pipeline that generates the SFT and DPO data is illustrated in \cref{fig:training-pipeline}.}

\begin{figure}
    \centering
    \includegraphics[width=1\linewidth]{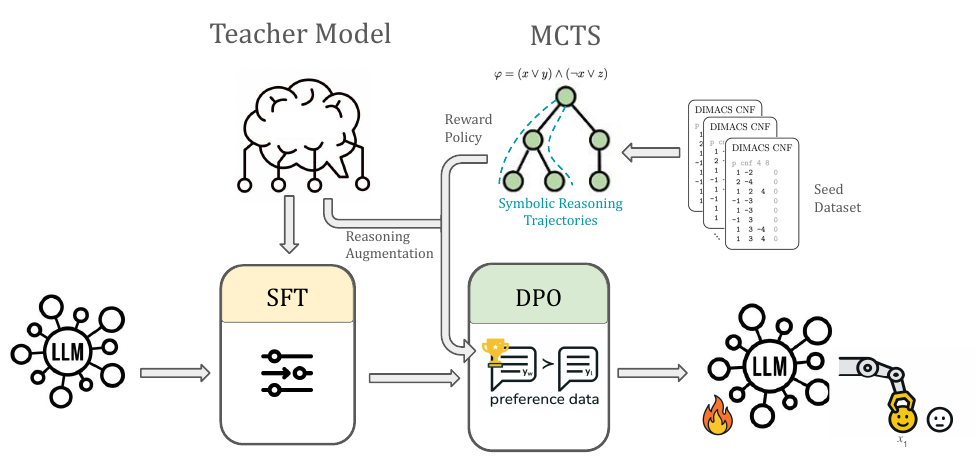}
    \caption{{\it Neuro-symbolic post-training pipeline}. A teacher model first creates SFT data for the first training stage. Monte-Carlo-Tree Search is then guided by symbolic cubing decisions to create DPO data, which is augmented with \revision{reasoning traces from the teacher model}.}
    \label{fig:training-pipeline}
\end{figure}

\begin{table}[t]
    \centering
    \small
    \caption{Training Hyperparameters}
    \label{tab:training-information}
    \begin{tabular}{l|l|l|l}
        \toprule
        \textbf{Model Name}         & \textbf{Qwen3-4B-DPO} & \textbf{Qwen3-4B-SFT} & \textbf{Qwen3-4B-SFT-DPO} \\
        \midrule
        Max Length                  & 8192                  & 17500                 & 8192                      \\
        Max Prompt Length           & 6144                  & -                     & 6144                      \\
        Per Device Train Batch Size & 1                     & 1                     & 1                         \\
        Per Device Eval Batch Size  & 8                     & -                     & 8                         \\
        Gradient Accumulation Steps & 8                     & 32                    & 8                         \\
        Number of Epochs            & 2                     & 2                     & 3                         \\
        Learning Rate               & 5e-6                  & 1e-5                  & 5e-6                      \\
        Optimizer                   & adamw\_8bit           & adamw\_torch\_fused   & adamw\_8bit               \\
        Weight Decay                & 0.1                   & 0.01                  & 0.1                       \\
        Learning Rate Scheduler     & cosine                & cosine                & cosine                    \\
        Warmup Ratio                & 0.1                   & 0.05                  & 0.1                       \\
        \midrule
        GPUs                        & 8 H100-80G            & 8 H100-80G            & 8 H100-80G                \\
        Duration                    & 3h 10m 11s            & 1h 6m 18s             & 8h 50m 30s                \\
        Checkpoint                  & final                 & final                 & \rebrev{final}            \\
        \bottomrule
    \end{tabular}
\end{table}

\subsection{\rebrev{Comparison with AlphaMapleSAT}}~\label{subsec:alphamaplesat}

\rebrev{AlphaMapleSAT~\citep{jha2024alphamaplesat} is the closest prior cubing system, an online Monte-Carlo-Tree-Search procedure with a non-transformer policy that runs from scratch on each instance at solve time. It contains no learned or persistent component, so it answers a different question than we do, namely how far per-instance search can go rather than whether a general cubing policy can be learned. Its online search scores candidates with a reward based on the unit-propagation signal alone, the $u$ term of our reward in \cref{eq:reward}, and expands them systematically over most of the available children. We instead run a similar MCTS procedure offline to build part of our DPO preference data rather than to cube at solve time, using a reward that additionally credits eliminated clauses, splitting decisions, and conflict clauses, expanding only up to six candidates guided by the symbolic and random heuristics. In its original MathCheck~\cite{li2024mathcheck} pipeline, AlphaMapleSAT generates cubes lazily and dispatches them to parallel solvers, whereas we emit a single cube and immediately solve the resulting subformula. Lastly, AlphaMapleSAT is CPU-only and fits its search to each benchmark separately, so unlike our trained policy it relies on no cross-instance generalization.}

\rebrev{We integrated AlphaMapleSAT's cubing tool into our own C\&C loop (\cref{alg:cube_and_conquer}), instead of its MathCheck pipeline, and evaluated it on the same 50 SAT and 50 UNSAT benchmarks. To match the queries of the other heuristics, we used search depth 1 so that each AlphaMapleSAT invocation returned a single cube. The results showed that AlphaMapleSAT solved 56 benchmarks per run (\cref{tab:extended-success-statistics}), the highest total of any heuristic we evaluated and above our best \passat{5} of 53. Although MCTS is stochastic in general, in our depth-1 setup it selected the same cube for every benchmark on all five runs, so its per-run totals are identical. This invariance is consistent with its systematic expansion and solver-computed scoring, which greatly reduces the run-to-run variation in the chosen cube, even though the downstream CDCL solving still differs between runs. Lastly, in our evaluation AlphaMapleSAT solved two UNSAT benchmarks that no other heuristic could.}

\rebrev{This stronger coverage comes from spending search at solve time, whereas our learned policy offers a single set of weights that decides without online search rather than through a per-instance search episode, transfers across instance families and difficulty levels, and produces a human-readable rationale for each cube (\cref{subsec:example-traces}). The two designs also compose rather than compete, since AlphaMapleSAT could serve as a stronger preference-data generator for our DPO stage while our richer reward could replace AlphaMapleSAT's propagation-only reward.}

\subsection{Extended Results}

\Cref{tab:extended-success-statistics} reports, for every heuristic, the number of benchmarks solved in each of the five runs, separately on the SAT and UNSAT halves and in total, together with the per-run mean and standard deviation. The \passat{k} scores of \cref{tab:pass_at_k_scores} are computed from these per-run results in the order they appear, and the figures that follow show the corresponding benchmark-level outcomes, namely the CDCL solve time of the solved benchmarks (\cref{fig:cnc-time-heatmap-extended}), and the total number of decisions made per benchmark (\cref{fig:total-decisions-heatmap-extended}), and which benchmarks each heuristic solves (\cref{fig:success-heatmap-extended}).

\begin{table}[ht]
    \centering
    \small
    \setlength{\tabcolsep}{4pt}
    \caption{Success statistics of different heuristics over 5 attempts on SAT and UNSAT benchmarks, sorted by average total rate. $\sigma$ denotes standard deviation across runs. \rebrev{OS denotes online search (AlphaMapleSAT).}}
    \label{tab:extended-success-statistics}
    \vspace{0.5em}
    \begin{tabular}{@{}ll ccccc ccccc ccccc cc@{}}
    \toprule
    & & \multicolumn{5}{c}{\textbf{SAT / 50}} & \multicolumn{5}{c}{\textbf{UNSAT / 50}} & \multicolumn{5}{c}{\textbf{Total / 100}} & & \\
    \cmidrule(lr){3-7} \cmidrule(lr){8-12} \cmidrule(lr){13-17}
    & \textbf{Heuristic} & \textbf{1} & \textbf{2} & \textbf{3} & \textbf{4} & \textbf{5} & \textbf{1} & \textbf{2} & \textbf{3} & \textbf{4} & \textbf{5} & \textbf{1} & \textbf{2} & \textbf{3} & \textbf{4} & \textbf{5} & \textbf{Avg} & $\boldsymbol{\sigma}$ \\
    \midrule
    \multirow{6}{*}{\rotatebox[origin=c]{90}{\scriptsize\textit{LLM}}}
    & Qwen3-4B          & 19 & 24 & 20 & 22 & 21 & 21 & 21 & 22 & 22 & 22 & 40 & 45 & 42 & 44 & 43 & 42.8 & 1.9 \\
    & GPT-OSS-120B      & 21 & 21 & 20 & 20 & 22 & 22 & 22 & 22 & 22 & 22 & 43 & 43 & 42 & 42 & 44 & 42.8 & 0.8 \\
    & Qwen3-32B         & 23 & 21 & 19 & 22 & 20 & 23 & 23 & 21 & 22 & 21 & 46 & 44 & 40 & 44 & 41 & 43.0 & 2.4 \\
    & GPT-OSS-20B       & 21 & 21 & 23 & 21 & 23 & 22 & 22 & 22 & 22 & 22 & 43 & 43 & 45 & 43 & 45 & 43.8 & 1.1 \\
    & Claude-3.7-Sonnet & 23 & 22 & 22 & 21 & 22 & 21 & 23 & 23 & 23 & 23 & 44 & 45 & 45 & 44 & 45 & 44.6 & 0.5 \\
    & Claude-Sonnet-4   & 24 & 26 & 23 & 22 & 24 & 23 & 23 & 23 & 23 & 23 & 47 & 49 & 46 & 45 & 47 & 46.8 & 1.5 \\
    \midrule
    \multirow{6}{*}{\rotatebox[origin=c]{90}{\scriptsize\textit{Symbolic}}}
    & random            & 20 & 22 & 24 & 23 & 24 & 24 & 23 & 23 & 23 & 22 & 44 & 45 & 47 & 46 & 46 & 45.6 & 1.1 \\
    & \zternary         & 26 & 26 & 27 & 26 & 27 & 25 & 25 & 25 & 25 & 25 & 51 & 51 & 52 & 51 & 52 & 51.4 & 0.5 \\
    & \zheuleschur      & 25 & 27 & 27 & 26 & 27 & 25 & 25 & 25 & 25 & 25 & 50 & 52 & 52 & 51 & 52 & 51.4 & 0.9 \\
    & \zunit            & 26 & 26 & 27 & 26 & 27 & 26 & 25 & 25 & 25 & 25 & 52 & 51 & 52 & 51 & 52 & 51.6 & 0.5 \\
    & \zmarchcu         & 26 & 27 & 26 & 27 & 27 & 25 & 25 & 25 & 25 & 25 & 51 & 52 & 51 & 52 & 52 & 51.6 & 0.5 \\
    & \zheuleu          & 27 & 27 & 27 & 27 & 26 & 25 & 25 & 25 & 25 & 25 & 52 & 52 & 52 & 52 & 51 & 51.8 & 0.4 \\
    \midrule
    \rebrev{\rotatebox[origin=c]{90}{\scriptsize\textit{OS}}} & \rebrev{AlphaMapleSAT} & \rebrev{27} & \rebrev{27} & \rebrev{27} & \rebrev{27} & \rebrev{27} & \rebrev{29} & \rebrev{29} & \rebrev{29} & \rebrev{29} & \rebrev{29} & \rebrev{56} & \rebrev{56} & \rebrev{56} & \rebrev{56} & \rebrev{56} & \rebrev{56.0} & \rebrev{0.0} \\
    \midrule
    \multirow{3}{*}{\rotatebox[origin=c]{90}{\scriptsize\textit{Ours}}}
    & Qwen3-4B-DPO      & 20 & 20 & 19 & 20 & \rebrev{21} & 22 & 21 & 21 & 23 & 22 & 42 & 41 & 40 & 43 & \rebrev{43} & \rebrev{41.8} & \rebrev{1.3} \\
    & Qwen3-4B-SFT      & 25 & 22 & 21 & 24 & 21 & 24 & 23 & 23 & 24 & 23 & 49 & 45 & 44 & 48 & 44 & 46.0 & 2.3 \\
    & Qwen3-4B-SFT-DPO  & 23 & 25 & 24 & 23 & 24 & 23 & 23 & 23 & 25 & 24 & 46 & 48 & 47 & 48 & 48 & 47.4 & 0.9 \\
    \bottomrule
    \end{tabular}
\end{table}

\begin{figure}
    \centering
    \includegraphics[width=\linewidth]{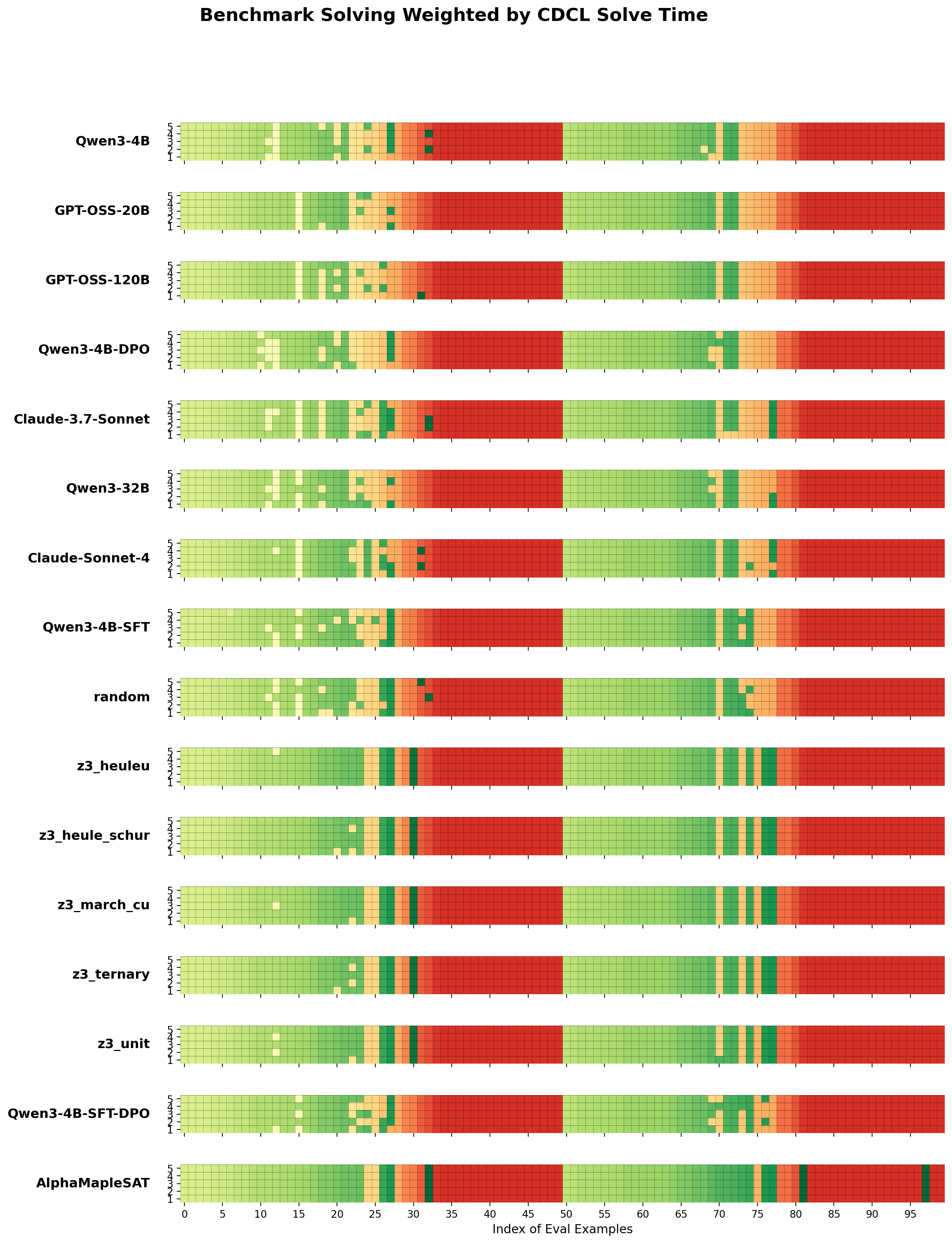}
    \caption{Solved benchmarks weighted and ordered by CDCL solve time [sat,unsat]. Green implies solved benchmark, red is unsolved ($<30min$). Deeper color denotes longer CDCL solve time.}
    \label{fig:cnc-time-heatmap-extended}
\end{figure}

\begin{figure}
    \centering
    \includegraphics[width=\linewidth]{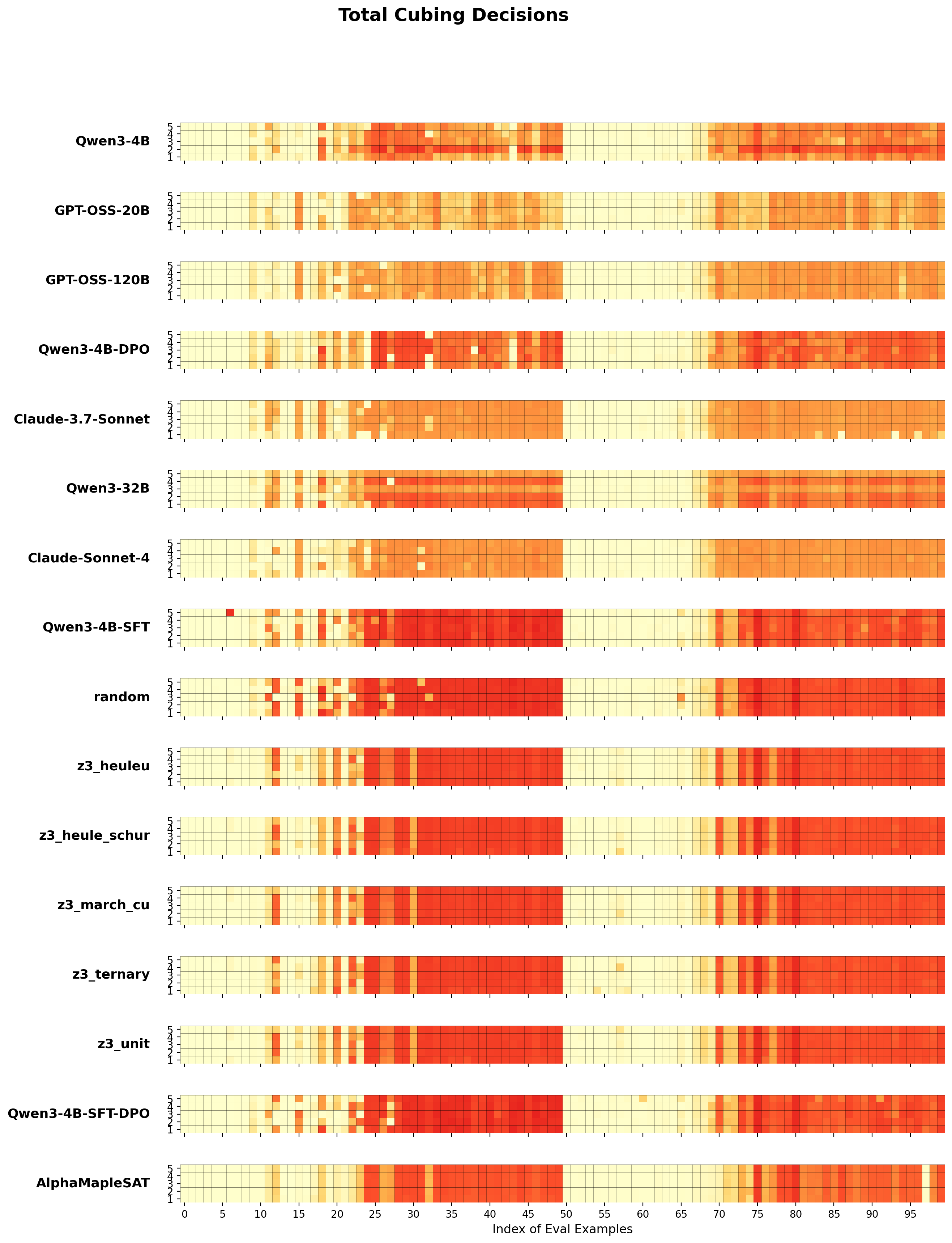}
    \caption{Total decisions made per benchmark (green implies fewer decisions). Symbolic heuristics and smaller models make more decisions in allotted time.}
    \label{fig:total-decisions-heatmap-extended}
\end{figure}

\begin{figure}
    \centering
    \includegraphics[width=0.96\linewidth]{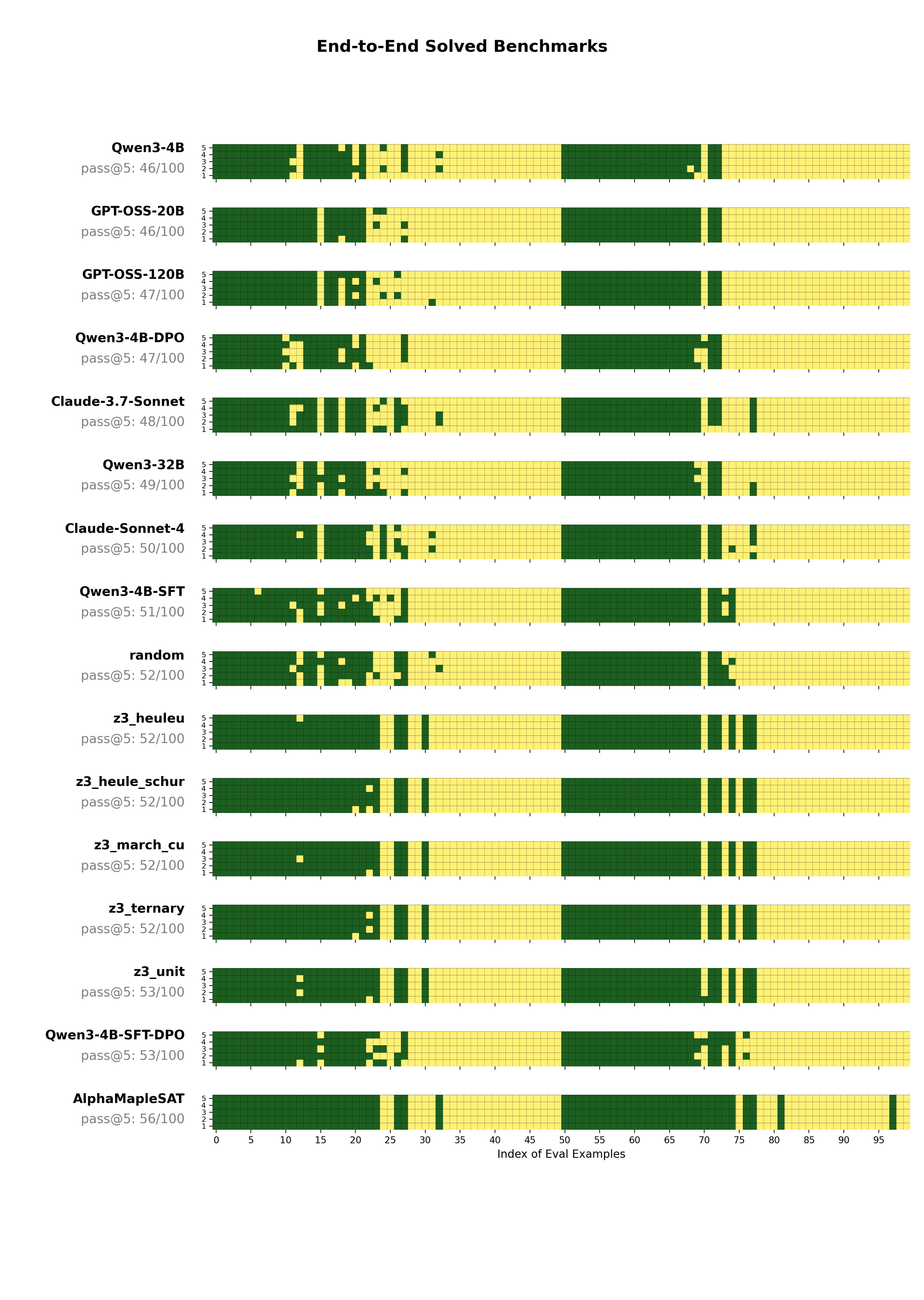}
    \caption{Fine-grained overview of successes (green) and unknowns (yellow) for each benchmark. First 50 benchmarks are SAT, last 50 are UNSAT. The heuristics are sorted based on \passat{5} score.}
    \label{fig:success-heatmap-extended}
\end{figure}

\subsection{\rebrev{Run-Order-Independent Pass@k}}~\label{subsec:run-order-independent-passk}

\rebrev{\Cref{tab:run-order-independent-passk} reports \passat{k}, the number of unique benchmarks solved within $k$ unordered attempts (runs), for every heuristic. Each value is averaged over all $\binom{5}{k}$ ways of choosing $k$ of the five evaluation runs, so that it charges a heuristic exactly $k$ attempts. The deterministic symbolic heuristics are nearly saturated after one attempt and gain at most 1.4 benchmarks up to $k{=}5$, whereas the LLMs and random cubing keep improving, by 2.2 to 6.4 benchmarks. Among our models, Qwen3-4B-SFT-DPO gains 5.6 (from 47.4 to 53.0), against 5.0 for Qwen3-4B-SFT and 3.2 for the untrained Qwen3-4B. Relative to \zunit, the gap of 4.2 benchmarks at one attempt shrinks to 1.9 at two attempts and 0.6 at three, and closes at four.}

\rebrev{\Cref{tab:pass_at_k_scores} reports a run-order-dependent \passat{k}, computed for one fixed ordering of the runs, in which the $k$ attempts are the first $k$ runs in the order they were executed. Its single-attempt column therefore reflects run 1 alone, where Qwen3-4B-SFT-DPO solved 46 benchmarks (against a mean of 47.4 over the five runs in \cref{tab:run-order-independent-passk}) and \zunit solved 52 (against a mean of 51.6), which makes the single-attempt gap 6 benchmarks instead of 4.2. Averaging over all subsets of $k$ runs removes this dependence on the ordering. For $k{=}5$ there is a single subset, so \passat{5} is identical in both tables.}

\begin{table}[ht]
    \centering
    \rebrev{
    \caption{\Passat{k} averaged over all $\binom{5}{k}$ subsets of the five evaluation runs, for $k=1,\dots,5$ (unique benchmarks solved out of 100). Sorted by the $k{=}1$ value within each group. \rebrev{OS denotes online search (AlphaMapleSAT).}}
    \label{tab:run-order-independent-passk}
    \vspace{0.5em}
    \begin{tabular}{@{}llccccc@{}}
        \toprule
        & \textbf{Heuristic} & \textbf{@1} & \textbf{@2} & \textbf{@3} & \textbf{@4} & \textbf{@5} \\
        \midrule
        \multirow{6}{*}{\rotatebox[origin=c]{90}{\scriptsize\textit{LLM}}}
        & Qwen3-4B & 42.8 & 44.7 & 45.4 & 45.8 & 46.0 \\
        & GPT-OSS-120B & 42.8 & 44.2 & 45.3 & 46.2 & 47.0 \\
        & Qwen3-32B & 43.0 & 45.6 & 47.2 & 48.2 & 49.0 \\
        & GPT-OSS-20B & 43.8 & 44.8 & 45.4 & 45.8 & 46.0 \\
        & Claude-3.7-Sonnet & 44.6 & 46.7 & 47.6 & 48.0 & 48.0 \\
        & Claude-Sonnet-4 & 46.8 & 48.6 & 49.4 & 49.8 & 50.0 \\
        \midrule
        \multirow{6}{*}{\rotatebox[origin=c]{90}{\scriptsize\textit{Symbolic}}}
        & random & 45.6 & 48.4 & 49.9 & 51.0 & 52.0 \\
        & \zternary & 51.4 & 51.9 & 52.0 & 52.0 & 52.0 \\
        & \zheuleschur & 51.4 & 51.9 & 52.0 & 52.0 & 52.0 \\
        & \zunit & 51.6 & 52.3 & 52.6 & 52.8 & 53.0 \\
        & \zmarchcu & 51.6 & 52.0 & 52.0 & 52.0 & 52.0 \\
        & \zheuleu & 51.8 & 52.0 & 52.0 & 52.0 & 52.0 \\
        \midrule
        \rotatebox[origin=c]{90}{\scriptsize\textit{OS}} & AlphaMapleSAT & 56.0 & 56.0 & 56.0 & 56.0 & 56.0 \\
        \midrule
        \multirow{3}{*}{\rotatebox[origin=c]{90}{\scriptsize\textit{Ours}}}
        & Qwen3-4B-DPO & 41.8 & 44.1 & 45.5 & 46.4 & 47.0 \\
        & Qwen3-4B-SFT & 46.0 & 48.7 & 49.9 & 50.6 & 51.0 \\
        & \textbf{Qwen3-4B-SFT-DPO} & \textbf{47.4} & \textbf{50.4} & \textbf{52.0} & \textbf{52.8} & \textbf{53.0} \\
        \bottomrule
    \end{tabular}
    }
\end{table}

\subsection{\rebrev{Performance Under Truncated Time Budgets}}~\label{subsec:budget-sweep}

\rebrev{Because the timeout only stops a run and does not influence its decisions, truncating a run at time $t$ is exact. \Cref{tab:budget-sweep} therefore recomputes \passat{5} and the per-run mean, neither of which depends on the order of the runs, for Qwen3-4B-SFT, Qwen3-4B-SFT-DPO and the strongest symbolic heuristic \zunit across per-attempt budgets. On \passat{5} our model is at or ahead of \zunit at every budget and gains most at the tightest one (43 against 38 at 5 minutes), because \zunit's remaining edge sits in a long tail of slow solves. On the per-run mean \zunit leads only from 15 minutes upward. Qwen3-4B-SFT-DPO is also at or ahead of Qwen3-4B-SFT at every budget on both measures, so the two-benchmark gain of DPO is not an artifact of the 30-minute threshold, although it remains small.}

\begin{table}[ht]
    \centering
    \rebrev{
    \caption{\passat{5} and per-run mean (avg) solved counts (out of 100) for Qwen3-4B-SFT, Qwen3-4B-SFT-DPO and the \zunit heuristic under truncated per-attempt time budgets, recomputed from the evaluation logs. Times include LLM decision latency. \textbf{Bold} marks the better value of Qwen3-4B-SFT-DPO and \zunit in each pairing.}
    \label{tab:budget-sweep}
    \vspace{0.5em}
    \begin{tabular}{@{}lcccccc@{}}
        \toprule
        & \multicolumn{2}{c}{\textbf{Qwen3-4B-SFT}} & \multicolumn{2}{c}{\textbf{Qwen3-4B-SFT-DPO}} & \multicolumn{2}{c}{\textbf{\zunit}} \\
        \cmidrule(lr){2-3} \cmidrule(lr){4-5} \cmidrule(lr){6-7}
        \textbf{Per-attempt budget} & \textbf{\passat{5}} & \textbf{avg} & \textbf{\passat{5}} & \textbf{avg} & \textbf{\passat{5}} & \textbf{avg} \\
        \midrule
        5 min  & 38 & 35.2 & \textbf{43} & \textbf{38.2} & 38 & 37.0 \\
        10 min & 43 & 39.4 & \textbf{47} & \textbf{42.0} & 42 & 41.8 \\
        15 min & 45 & 41.4 & \textbf{47} & 43.0 & 45 & \textbf{44.2} \\
        20 min & 48 & 42.6 & 49 & 44.2 & 49 & \textbf{46.2} \\
        30 min & 51 & 46.0 & 53 & 47.4 & 53 & \textbf{51.6} \\
        \bottomrule
    \end{tabular}
    }
\end{table}

\subsection{\rebrev{Mixed Two-Attempt Portfolios}}~\label{subsec:mixed-portfolio}

\rebrev{Cube-and-Conquer on hard instances is often deployed as a racing portfolio of configurations, so we ask whether the complementary coverage of \cref{sec:results} specifically requires the learned policy or instead follows from simply running a second, differently seeded attempt. Reusing the same five evaluation runs per method, \cref{tab:mixed-portfolio} reports the mean unique benchmarks solved by a two-attempt portfolio that pairs one \zunit run with one run of each partner. Pairing \zunit with another deterministic symbolic heuristic stays at or below two \zunit runs (52.08 to 52.12 against 52.30), whereas most stochastic partners improve on it, namely random cubing (52.80), Qwen3-4B-SFT-DPO (52.96), and Claude-Sonnet-4 (53.44), which is the highest among them. Several learned models do not, namely GPT-OSS-20B, Qwen3-32B, and Qwen3-4B-DPO. The learned policy therefore adds only about 0.16 benchmarks over random cubing in this retrospective two-attempt view, and trails Claude-Sonnet-4 by about 0.5, so most of the two-attempt gain reflects decision diversity rather than our neural policy in particular. AlphaMapleSAT reaches 57, but this follows from it already solving 56 benchmarks on its own and not from complementarity with \zunit. These pairings reuse the same five runs per method, so they are not independent trials and the small differences between partners should be read as indicative only.}

\begin{table}[ht]
    \centering
    \rebrev{
    \caption{Mean unique benchmarks solved (out of 100) by a two-attempt portfolio pairing one \zunit run with one run of each partner, reusing the same five evaluation runs per method and averaging over all available pairings. \emph{Pairings}: number of distinct run pairs averaged. \textbf{Bold} marks our model. \rebrev{OS denotes online search (AlphaMapleSAT).}}
    \label{tab:mixed-portfolio}
    \vspace{0.5em}
    \begin{tabular}{@{}ll cccc@{}}
        \toprule
        & \textbf{\zunit + partner} & \textbf{Mean solved} & \textbf{Min} & \textbf{Max} & \textbf{Pairings} \\
        \midrule
        \multirow{6}{*}{\rotatebox[origin=c]{90}{\scriptsize\textit{LLM}}}
        & Qwen3-4B             & 52.40 & 51 & 54 & 25 \\
        & GPT-OSS-120B         & 52.40 & 52 & 53 & 25 \\
        & GPT-OSS-20B          & 52.20 & 52 & 53 & 25 \\
        & Claude-3.7-Sonnet    & 52.72 & 51 & 53 & 25 \\
        & Qwen3-32B            & 51.92 & 51 & 54 & 25 \\
        & Claude-Sonnet-4      & 53.44 & 53 & 55 & 25 \\
        \midrule
        \multirow{6}{*}{\rotatebox[origin=c]{90}{\scriptsize\textit{Symbolic}}}
        & random               & 52.80 & 51 & 55 & 25 \\
        & \zternary            & 52.12 & 52 & 53 & 25 \\
        & \zheuleschur         & 52.12 & 52 & 53 & 25 \\
        & \zunit (repeat)      & 52.30 & 51 & 53 & 10 \\
        & \zmarchcu            & 52.08 & 51 & 53 & 25 \\
        & \zheuleu             & 52.12 & 51 & 53 & 25 \\
        \midrule
        \rotatebox[origin=c]{90}{\scriptsize\textit{OS}} & AlphaMapleSAT        & 57.00 & 57 & 57 & 25 \\
        \midrule
        \multirow{3}{*}{\rotatebox[origin=c]{90}{\scriptsize\textit{Ours}}}
        & Qwen3-4B-DPO         & 51.88 & 51 & 53 & 25 \\
        & Qwen3-4B-SFT         & 52.56 & 51 & 54 & 25 \\
        & Qwen3-4B-SFT-DPO     & \textbf{52.96} & 52 & 54 & 25 \\
        \bottomrule
    \end{tabular}
    }
\end{table}

\subsection{SAT vs UNSAT Performance Analysis}~\label{subsec:sat-vs-unsat}
\label{subsec:sat-bias}

\revision{To investigate whether heuristics exhibit bias toward SAT or UNSAT instances, we perform paired t-tests comparing each heuristic's SAT performance (out of 50) versus UNSAT performance (out of 50) across 5 runs. \Cref{tab:sat-unsat-ttest} reports the results.}

\begin{table}[ht]
    \centering
    \revision{
        \caption{Paired t-test of SAT vs UNSAT performance per heuristic (5 runs). Significant results ($p < 0.05$) are bolded. \rebrev{The test is undefined (n/a) for AlphaMapleSAT because it solved the same benchmarks in all five runs, so its paired difference has zero variance. OS denotes online search (AlphaMapleSAT).}}
        \label{tab:sat-unsat-ttest}
        \vspace{0.5em}
        \begin{tabular}{@{}ll ccccl@{}}
            \toprule
            & \textbf{Heuristic} & \textbf{SAT avg} & \textbf{UNSAT avg} & \textbf{Diff} & \textbf{$p$-value} & \textbf{Bias} \\
            \midrule
            \multirow{6}{*}{\rotatebox[origin=c]{90}{\scriptsize\textit{LLM}}}
            & Qwen3-4B           & 21.2             & 21.6               & $-0.4$        & \rebrev{0.688}              &               \\
            & GPT-OSS-120B       & 20.8             & 22.0               & $-1.2$        & \textbf{0.033}     & UNSAT         \\
            & \rebrev{GPT-OSS-20B}        & \rebrev{21.8}             & \rebrev{22.0}               & \rebrev{$-0.2$}        & \rebrev{0.704}              &               \\
            & Claude-3.7-Sonnet  & 22.0             & 22.6               & $-0.6$        & 0.426              &               \\
            & Qwen3-32B          & 21.0             & 22.0               & $-1.0$        & 0.089              &               \\
            & Claude-Sonnet-4    & 23.8             & 23.0               & $+0.8$        & 0.294              &               \\
            \midrule
            \multirow{6}{*}{\rotatebox[origin=c]{90}{\scriptsize\textit{Symbolic}}}
            & random             & 22.6             & 23.0               & $-0.4$        & 0.717              &               \\
            & \rebrev{\zternary}         & \rebrev{26.4}             & \rebrev{25.0}               & \rebrev{$+1.4$}        & \rebrev{\textbf{0.005}}     & \rebrev{SAT}           \\
            & \zheuleschur       & 26.4             & 25.0               & $+1.4$        & \textbf{0.025}     & SAT           \\
            & \zmarchcu          & 26.6             & 25.0               & $+1.6$        & \textbf{0.003}     & SAT           \\
            & \zunit               & 26.4             & 25.2               & $+1.2$        & \textbf{0.033}     & SAT           \\
            & \rebrev{\zheuleu}          & \rebrev{26.8}             & \rebrev{25.0}               & \rebrev{$+1.8$}        & \rebrev{\textbf{0.001}}     & \rebrev{SAT}           \\
            \midrule
            \rebrev{\rotatebox[origin=c]{90}{\scriptsize\textit{OS}}} & \rebrev{AlphaMapleSAT}     & \rebrev{27.0}             & \rebrev{29.0}               & \rebrev{$-2.0$}        & \rebrev{n/a}               &               \\
            \midrule
            \multirow{3}{*}{\rotatebox[origin=c]{90}{\scriptsize\textit{Ours}}}
            & Qwen3-4B-DPO       & \rebrev{20.0}             & 21.8               & \rebrev{$-1.8$}        & \rebrev{\textbf{0.009}}     & UNSAT         \\
            & Qwen3-4B-SFT       & 22.6             & 23.4               & $-0.8$        & 0.242              &               \\
            & Qwen3-4B-SFT-DPO   & 23.8             & 23.6               & $+0.2$        & 0.778              &               \\
            \bottomrule
        \end{tabular}
    }
\end{table}

\revision{Out of the sixteen heuristics, eight show no significant difference between SAT and UNSAT performance, \rebrev{seven show a small but significant bias of between 1.2 and 1.8 benchmarks out of 50, and the test is undefined for AlphaMapleSAT.} All five deterministic symbolic heuristics lean towards SAT instances, which suggests that their scoring functions are tuned to the structure of satisfiable formulas. \rebrev{AlphaMapleSAT cannot be tested this way since its results do not vary across the five runs, and it solves two more UNSAT than SAT benchmarks.} On the learned side, GPT-OSS-120B and the DPO-only variant (no SFT) show a slight UNSAT bias. Notably, our main model Qwen3-4B-SFT-DPO is well-balanced (SAT=23.8, UNSAT=23.6, $p=0.778$), suggesting that the combination of SFT and DPO produces a heuristic without systematic category bias.} \finalrev{The SAT bias of the symbolic lookahead heuristics is consistent with the test-set composition reported in \cref{tab:benchmark-families}, as the SAT split is 74\% hard random 3-SAT variants (\emph{sgen}, Barthel, random-graph), a class where lookahead heuristics like these typically outperform CDCL alone.}

\subsection{\finalrev{Benchmark Family Distribution}}~\label{subsec:benchmark-families}

\finalrev{\Cref{tab:benchmark-families} reports the per-family breakdown of the 100 test benchmarks classified by SAT-Competition naming convention. Families were inferred from the Global Benchmark Database filename prefix. The distribution is weighted toward hard random 3-SAT (34\%, mostly \emph{sgen}) and XOR / modular-arithmetic (19\%), the instance classes for which Cube-and-Conquer was originally developed~\citep{heule2011cube}. The SAT split is more concentrated ($28/50$ \emph{sgen} + 6 Barthel + 3 random-graph = 74\% hard-random) than the UNSAT split, which spans ten families without any dominating above 18\%.}

\begin{table}[ht]
    \centering
    \small
    \caption{\finalrev{Per-family distribution of the 100 test benchmarks, inferred from SAT-Competition filename prefixes.}}
    \label{tab:benchmark-families}
    \vspace{0.4em}
    \begin{tabular}{@{}lrrr@{}}
\toprule
\textbf{Family} & \textbf{SAT} & \textbf{UNSAT} & \textbf{Total} \\
\midrule
Hard random 3-SAT (sgen / hgen) & 28 & 6 & 34 \\
XOR / modular-arithmetic & 11 & 8 & 19 \\
Graph encodings (edges / fixedbandwidth / random graphs) & 3 & 9 & 12 \\
Urquhart (constraint graph) & 1 & 9 & 10 \\
Combinatorial puzzles (pigeonhole / battleship) & 1 & 5 & 6 \\
Frustrated-loops (Barthel / glassy) & 6 & 0 & 6 \\
Cryptographic (Hidden Weighted Bit) & 0 & 5 & 5 \\
Margulis (expander graph) & 0 & 4 & 4 \\
ISCAS-like circuit & 0 & 3 & 3 \\
Model-counting derived & 0 & 1 & 1 \\
\midrule
\textbf{Total} & 50 & 50 & 100 \\
\bottomrule
\end{tabular}

\end{table}

\subsection{\rebrev{Decision-Time Accounting}}~\label{subsec:decision-time-accounting}

\rebrev{\Cref{tab:decision-time-accounting} aggregates, over the five runs and 100 benchmarks, the end-to-end C\&C time, the time spent deciding, and the share of the former taken by the latter. The total C\&C time varies little across heuristics (between 143 and 159 hours, and 128 hours for AlphaMapleSAT, which solves more benchmarks), whereas the decision time spans a factor of about 26, from 3.4 hours for \zmarchcu to 87.4 hours for GPT-OSS-20B. The share therefore tracks decision time, ranging from 2.4\% for the symbolic heuristics to 56.5\% for GPT-OSS-20B. Among our models, the share falls from 39.5\% for the untrained Qwen3-4B to 10.9\% for Qwen3-4B-SFT and 9.4\% for Qwen3-4B-SFT-DPO, against 3.5\% for \zunit.}

\rebrev{The table complements the decision-time figures in two respects. \Cref{fig:accumulated-decision-time}, which is the subset of \cref{fig:accumulated-decision-time-extended} shown in the main text, charges each heuristic for one attempt only, its best SAT and best UNSAT runs, whereas the table covers all five runs of every heuristic in the full figure. Both measure decision time as the C\&C time minus the solver time, so they agree on the ordering of the heuristics in the main-text figure and on the roughly 2.8-fold gap between Qwen3-4B-SFT-DPO and \zunit (8.5 versus 3.0 thousand seconds in the figures, 13.7 versus 4.9 hours in the table). Absolute values, and some ratios between heuristics, nevertheless differ between the figures and the table, since the figures cover only the selected attempts.}

\begin{figure}
    \centering
    \includegraphics[width=0.8\linewidth]{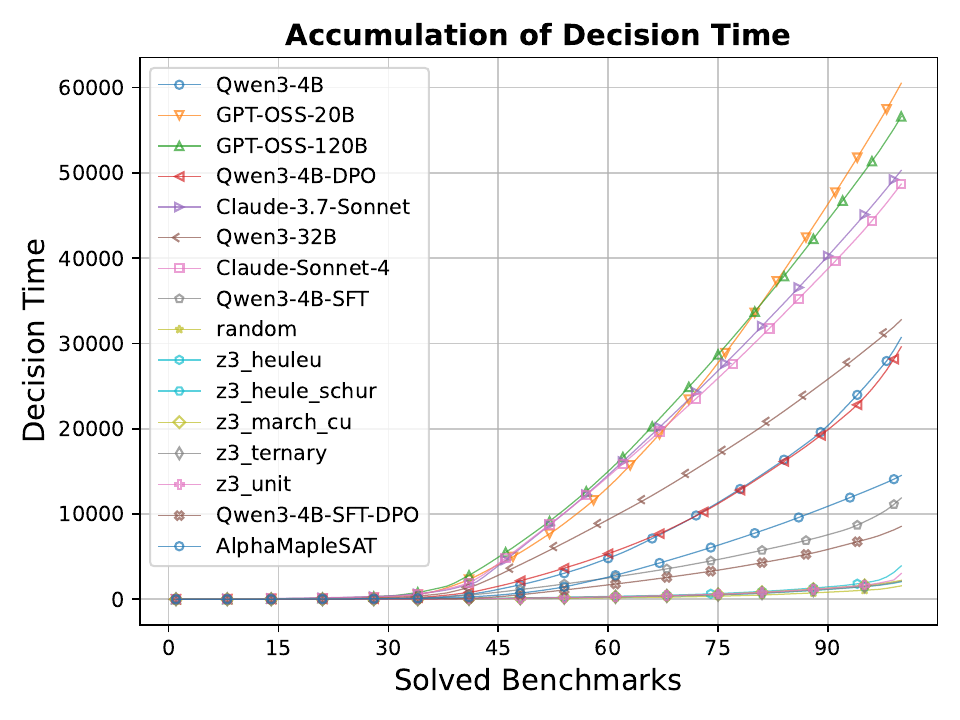}
    \caption{Accumulated decision time (sec) per heuristic\rebrev{, using each heuristic's single best SAT and best UNSAT attempts (\cref{tab:extended-success-statistics}). The metric is total time spent deciding during the end-to-end C\&C run, with benchmarks sorted and stacked by decision time, so the curve reads as the number of benchmarks solved within a given decision-time budget. \Cref{tab:decision-time-accounting} aggregates the decision time over all five runs}. Larger models need more time to decide and respond.}
    \label{fig:accumulated-decision-time-extended}
\end{figure}

\begin{table}[ht]
    \centering
    \rebrev{
    \caption{Decision-time accounting over all 5 runs of the 100 benchmarks, in hours. For benchmark $b$ and run $r$, the decision time is $t^{b,r}_{\mathrm{dec}} = t^{b,r}_{\mathrm{C\&C}} - t^{b,r}_{\mathrm{sat}}$, the end-to-end C\&C time minus the solver time. $T_{\mathrm{C\&C}}$ and $T_{\mathrm{dec}}$ sum $t_{\mathrm{C\&C}}$ and $t_{\mathrm{dec}}$ over all benchmarks and runs, and Share is the ratio $T_{\mathrm{dec}}/T_{\mathrm{C\&C}}$, the fraction of the C\&C time spent deciding. For AlphaMapleSAT, decision time is its online MCTS search. Heuristics are sorted by $T_{\mathrm{dec}}$.}
    \label{tab:decision-time-accounting}
    \vspace{0.5em}
    \begin{tabular}{@{}lrrr@{}}
        \toprule
        \textbf{Heuristic} & \textbf{$\boldsymbol{T_{\mathrm{C\&C}}}$ (h)} & \textbf{$\boldsymbol{T_{\mathrm{dec}}}$ (h)} & \textbf{Share} \\
        \midrule
        GPT-OSS-20B          & 154.7 & 87.4 & 56.5\% \\
        GPT-OSS-120B         & 155.5 & 79.6 & 51.2\% \\
        Claude-3.7-Sonnet    & 158.9 & 76.4 & 48.1\% \\
        Claude-Sonnet-4      & 150.9 & 70.0 & 46.4\% \\
        Qwen3-4B (base)      & 156.7 & 62.0 & 39.5\% \\
        Qwen3-32B            & 156.5 & 61.7 & 39.4\% \\
        Qwen3-4B-DPO         & 150.2 & 42.2 & 28.1\% \\
        AlphaMapleSAT        & 127.6 & 17.9 & 14.0\% \\
        Qwen3-4B-SFT         & 150.3 & 16.4 & 10.9\% \\
        \textbf{Qwen3-4B-SFT-DPO} & \textbf{145.5} & \textbf{13.7} & \textbf{9.4\%} \\
        \zternary              & 142.9 &  5.1 &  3.6\% \\
        \zunit          & 142.8 &  4.9 &  3.5\% \\
        \zheuleschur         & 142.7 &  4.9 &  3.4\% \\
        random               & 150.9 &  4.0 &  2.7\% \\
        \zheuleu               & 142.7 &  3.5 &  2.4\% \\
        \zmarchcu            & 142.6 &  3.4 &  2.4\% \\
        \bottomrule
    \end{tabular}
    }
\end{table}

\subsection{Quantitative Analysis of Decision Behavior}~\label{subsec:quantitative-behavior}

\revision{The heuristic classification in \cref{tab:heuristic-classification} summarises \emph{what the model says} it is doing, based on a post-hoc LLM-judge reading of the reasoning text. In this section we examine \emph{what the model actually does}, by parsing the DIMACS prompt of each cubing decision and comparing the chosen variable against structural properties of the formula. All analyses focus on the first-cube (root) decision, the only decision point where every heuristic faces an identical formula across models, and use the 100 test benchmarks (50 SAT, 50 UNSAT) across 5 evaluation runs. In addition to our main model, we include baselines from every heuristic reported in \cref{tab:pass_at_k_scores}, namely the Qwen3-4B base model, the SFT-only and DPO-only ablations, two teacher candidates (GPT-OSS-20B, GPT-OSS-120B), a larger LLM baseline (Qwen3-32B), and six symbolic heuristics (\zunit, \zheuleschur, \zmarchcu, \zheuleu, \zternary, \emph{random}).}

\subsubsection{Do the chosen variables match the dominant heuristic (H1)?}~\label{subsec:freq-rank}

\revision{We count the clause-occurrence frequency of every variable in each CNF and rank the chosen first-cube variable by that count. A rank percentile of 0 means the chosen variable is the most frequent in the formula, and a rank percentile of 0.5 is the median.} \rebrev{Because the rank counts only variables with strictly higher frequency, every variable tied at the top frequency receives rank 0, so a uniformly-random choice averages 0.5 only when all frequencies are distinct. These formulas are instead dominated by ties, with a mean of 63\% and a median of 99\% of variables sharing the top clause-occurrence frequency, so the tie-aware uniform reference is far lower and lands in the top 10\% for 74\% of its picks with a mean rank of 0.13. This tie-aware reference is the baseline for reading the top-10\% rates below.}

\revision{\Cref{fig:var-freq-rank} shows the distribution for Qwen3-4B-SFT-DPO alone. \rebrev{88\%} of SAT first-cube choices and \rebrev{70\%} of UNSAT first-cube choices lie in the top 10\% most-frequent variables (\rebrev{median rank percentile 0.000 for both, overall mean 0.096}). Taken in isolation, this number appears to confirm the H1 (Variable Frequency Analysis) majority in \cref{tab:heuristic-classification} at the level of realised decisions. However, the same measurement on the baselines and symbolic heuristics (\cref{tab:first-cube-comparison}) reveals that the top-10\% behaviour is essentially universal across LLMs, since the untrained Qwen3-4B base sits at 79\%, the GPT-OSS-120B teacher at 90\%, and even \emph{random} reaches 74\%\rebrev{, which matches the tie-aware uniform reference established above and shows that the top-10\% rate reflects how these formulas place most variables at the same top frequency rather than any selective behaviour}. Symbolic lookahead heuristics explicitly designed around clause weighting (\zunit, \zheuleschur, \zmarchcu, etc.) cluster at 85--86\%. We therefore read this metric as a \emph{calibration result}, since all LLMs we measured and our trained model in particular operate well inside the same ``top-frequency tail'' as established symbolic heuristics, rather than as evidence of a behaviour that post-training introduced.}

\begin{figure}[ht]
\centering
\includegraphics[width=0.9\linewidth]{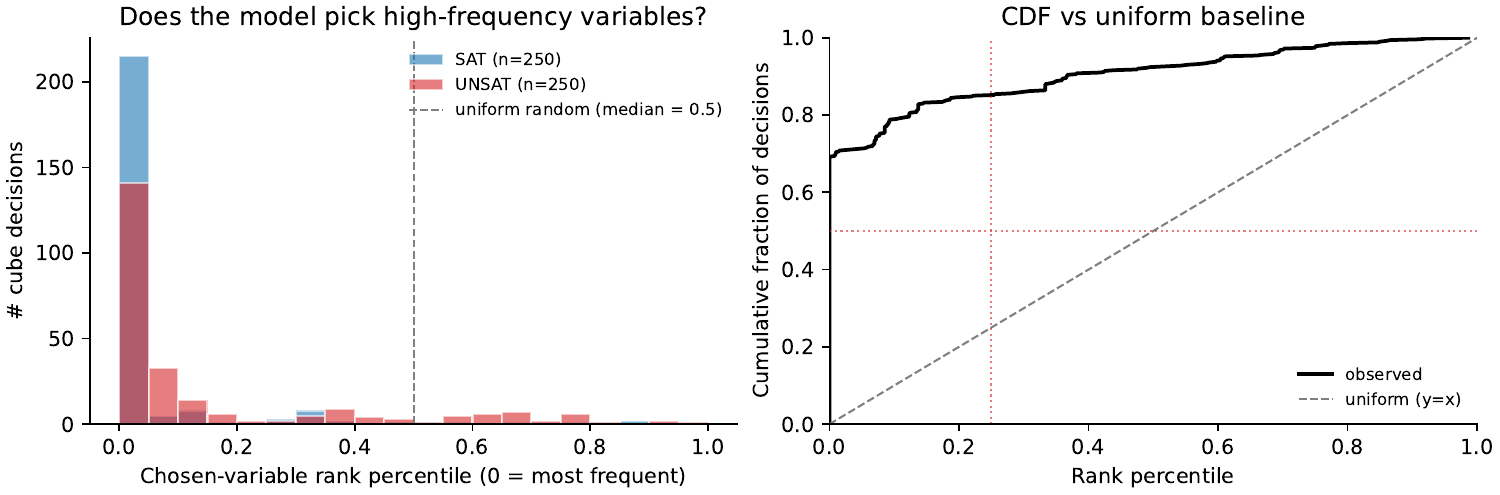}
\caption{\revision{Clause-occurrence rank percentile of the first-cube variable chosen by Qwen3-4B-SFT-DPO, pooled over 5 runs $\times$ 100 benchmarks (250 SAT + 250 UNSAT decisions). \emph{Left}: histogram. \emph{Right}: cumulative distribution versus a uniform-random baseline. \rebrev{88\%} of SAT and \rebrev{70\%} of UNSAT choices lie in the top 10\% of variables by clause frequency. \Cref{tab:first-cube-comparison} shows that this behavior is largely shared by other LLMs and by symbolic baselines.}}
\label{fig:var-freq-rank}
\end{figure}

\subsubsection{Where does solution diversity come from? A cross-heuristic comparison}~\label{subsec:diversity}

\revision{\Cref{tab:extended-success-statistics} shows that Qwen3-4B-SFT-DPO solves 53 unique benchmarks across 5 runs and averages 47.4 per run. The numbers gap is evidence that different runs solve different benchmarks. We quantify the diversity that produces this gap by (i) measuring the Shannon entropy of the first-cube variable across the 5 runs per benchmark (upper bound $\log_2 5 \approx 2.32$ bits), and (ii) computing the mean pairwise run agreement, namely the fraction of benchmarks on which two different runs pick the same first-cube variable.} \rebrev{This ceiling is the maximum entropy attainable from five samples, reached only when all five runs pick distinct variables, and is not the maximum entropy of the underlying policy over the full variable set. We estimate entropy per benchmark from these at-most-five samples and include only benchmarks on which at least two runs produced a parseable first cube, which is all 100 benchmarks for every heuristic except the DPO-only ablation, where one benchmark produced no parseable first cube in any of the five runs and 99 remain.} \revision{\finalrev{All LLMs in this comparison share an identical inference configuration (same temperature and top-$p$ settings on the same vLLM server), so any diversity differences between models reflect training, not decoding-time sampling.} Contrasting our model against the same set of baselines used in \cref{subsec:freq-rank} isolates which training stage unlocks diversity.}

\revision{\Cref{tab:first-cube-comparison} reports both axes for every heuristic, and \cref{fig:mechanism} in the main body plots them together. Three clear regimes emerge.}

\revision{\textbf{Symbolic heuristics} (\zunit, \zheuleschur, \zmarchcu, \zheuleu, \zternary) cluster tightly at mean entropy 0.23--0.25 bits and run agreement 86--88\%. Apart from tie-breaking, these heuristics are deterministic, which bounds how much portfolio benefit any repeated run can extract from them. \emph{random} sits at the opposite extreme (entropy 2.31, agreement 0\%).}

\revision{\textbf{Base and DPO-only LLMs} behave almost symbolically, as the untrained Qwen3-4B has entropy 0.55 and agreement 70\%, and the DPO-only ablation (Qwen3-4B-DPO) is nearly identical at 0.59 / 68\%. These models are near-deterministic at the root in the sense that their temperature-sampled decodes usually resolve to the same first cube across runs.}

\revision{\textbf{SFT-based LLMs} are genuinely diverse. Qwen3-4B-SFT and Qwen3-4B-SFT-DPO reach entropy 1.93 and 1.78 with run agreement 12\% and 18\% respectively, which is more diverse than either the 20B teacher candidate (1.29 / 36\%), the 120B teacher (1.13 / 44\%), or the 32B baseline (1.17 / 42\%). DPO after SFT slightly narrows the diversity (1.93 $\rightarrow$ 1.78) without collapsing it.} \rebrev{This is a measured effect in our setup and not a general property, since preference optimization does not necessarily reduce entropy and can even raise it for suitable reference and reward values. One hypothesis consistent with the small narrowing we observe is that the DPO objective concentrates probability on high-reward variables while leaving meaningful mass on alternatives.}

\rebrev{Outside these three regimes, AlphaMapleSAT picks the same first-cube variable on all five runs for every benchmark (entropy 0.00, agreement 100\%), which is more extreme than the symbolic heuristics. It also chooses less frequent variables than they do, with 71\% of its choices in the top 10\% of variables against 85--86\% for the symbolic scoring heuristics. This is consistent with its search scoring candidates by a solver-computed propagation signal and not by clause frequency.}

\revision{Because base-only and DPO-only ablations share the base model's near-deterministic root behaviour, SFT is the training stage that produces the diversity observed in our main model. This gives a specific interpretation of why Qwen3-4B-SFT-DPO contributes complementary coverage in a portfolio setting while symbolic heuristics cannot. The symbolic heuristics' agreement ceiling of 86--88\% means five independent runs mostly re-solve the same benchmarks, whereas our model's 18\% agreement spreads five runs over largely disjoint first moves.}

\rebrev{This leaves open why SFT raises root-level diversity and DPO does not. We offer the following account as a hypothesis that is consistent with these measurements and not as a tested mechanism. The SFT data consists of the teacher's answers along whole C\&C trajectories. Many intermediate formulas are structurally similar, and in such states several cubes are almost equally good, so the teacher, which decodes with sampling, answers differently on similar prompts. Cross-entropy training on these answers teaches the student the teacher's marginal distribution over good cubes, which is flatter than any single teacher run because the decisions within one run are correlated through its autoregressive reasoning prefix. This is consistent with the student's first-cube entropy of 1.78 to 1.93 bits exceeding the teacher's own 1.13 bits. DPO instead concentrates probability on preferred completions under a KL anchor to its reference policy. When the reference is near-deterministic at the root, as the base model is with 0.55 bits, DPO reweights existing modes and does not create new ones, which is consistent with the DPO-only ablation staying at 0.59 bits. The same view fits the drop in entropy from 1.93 to 1.78 bits when DPO follows SFT, since preference optimization is known to be mode-seeking.}

\finalrev{\Cref{fig:agreement-vs-gain} tests the connection between root-level diversity and portfolio gain directly. For each heuristic in \cref{tab:pass_at_k_scores} \rebrev{and for AlphaMapleSAT}, we plot its mean pairwise run agreement against its \emph{portfolio gain}, defined as the difference between \passat{5} (the union of benchmarks solved across 5 runs) and the per-run mean (the average number of benchmarks solved in a single run). Portfolio gain measures the extra benchmarks that a 5-run portfolio recovers over a single run. For example, Qwen3-4B-SFT-DPO solves 47.4 benchmarks per run on average but 53 unique benchmarks across 5 runs, for a gain of $+5.6$, while the symbolic \zunit heuristic solves 51.6 per run and 53 across 5 runs, for a gain of only $+1.4$. Both reach the same \passat{5}, but through different mechanisms, with \zunit solving a strong, stable set every run and Qwen3-4B-SFT-DPO solving different, complementary sets across runs. The relationship is negative (Pearson \rebrev{$r=-0.83$} across the 13 heuristics of \cref{tab:pass_at_k_scores}\rebrev{, and $r=-0.85$ with AlphaMapleSAT added}), so heuristics with lower root-level agreement tend to extract more benefit from additional attempts. This supports the claim that the root-level diversity produced by SFT is what translates into the complementary per-run coverage observed in the main body.}

\rebrev{First-cube agreement only lower-bounds the diversity of a full trajectory for the LLMs, since two runs can choose the same root variable and still diverge deeper in the tree. To test this, we look at the benchmarks that a model solves in some of its five runs but not in all of them, which we call \emph{flip} benchmarks. A benchmark solved in every run or in none adds nothing to the portfolio gain, so the flip benchmarks are the only ones that create it. \Cref{tab:flip-benchmarks} reports two ratios for each model, each out of its own total in the header. The first is on how many flip benchmarks a solved and an unsolved run chose the same first cube, which tests whether the first decision explains the difference between runs. The second is how many of the solved flip runs finish in the last ten minutes, which tests whether the gain depends on solves close to the timeout. Nine unsolved runs of the DPO-only ablation have no recorded first cube and are not counted as sharing one. The first count is 8 of 8 for the base model and 6 of 9 for the DPO-only ablation, so their gain arises despite identical root decisions, from sampling stochasticity at deeper levels and timing variance on instances near the timeout. The second count agrees, since 11 of the 19 solved flip runs of the DPO-only ablation (58\%) finish in the last ten minutes. For Qwen3-4B-SFT and Qwen3-4B-SFT-DPO the first count is only 2 of 11 and 3 of 11, so solved and unsolved runs almost always differ already at the root. This is why the DPO-only ablation sits above the regression line of \cref{fig:agreement-vs-gain}.}

\begin{figure}[ht]
\centering
\includegraphics[width=0.75\linewidth]{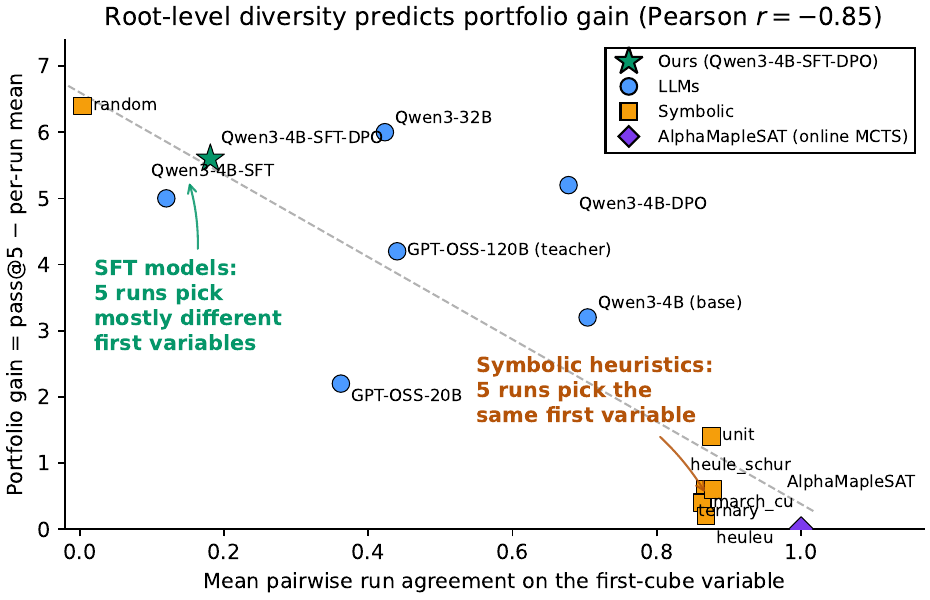}
\caption{\finalrev{Root-level decision diversity versus portfolio gain, one point per heuristic in \cref{tab:pass_at_k_scores}\rebrev{ and for AlphaMapleSAT}. \emph{Horizontal}: mean pairwise run agreement on the first-cube variable (0 = every pair of runs picks a different variable, 1 = every pair picks the same variable). \emph{Vertical}: \passat{5} $-$ per-run mean, i.e.\ the extra benchmarks a 5-run portfolio recovers over a single run. Heuristics with lower agreement tend to gain more from repeated attempts (Pearson \rebrev{$r = -0.85$ with AlphaMapleSAT, $-0.83$ without}). Symbolic heuristics cluster in the deterministic / low-gain corner (5 runs re-solve the same set), while SFT-based LLMs occupy the diverse / high-gain corner (5 runs solve different, complementary sets).}}
\label{fig:agreement-vs-gain}
\end{figure}

\begin{table}[ht]
\centering
\caption{\rebrev{Flip benchmarks, which are solved in between one and four of the five runs. Each entry reads as a count out of a total. \emph{Shared first cube / flips}: flip benchmarks on which at least one solved run and one unsolved run chose the same first-cube variable, out of all flip benchmarks. \emph{Late solves / solved flip runs}: solved flip runs that finish in the last ten minutes of the 30-minute timeout, out of all runs that solve a flip benchmark.}}
\label{tab:flip-benchmarks}
\begin{tabular}{@{}lcc@{}}
\toprule
\textbf{Model} & \shortstack{\textbf{Shared first cube}\\\textbf{/ flips}} & \shortstack{\textbf{Late solves}\\\textbf{/ solved flip runs}} \\
\midrule
Qwen3-4B & 8/8 & 10/24 \\
Qwen3-4B-DPO & 6/9 & 11/19 \\
Qwen3-4B-SFT & 2/11 & 9/30 \\
Qwen3-4B-SFT-DPO & 3/11 & 10/27 \\
\bottomrule
\end{tabular}
\end{table}

\begin{table}[ht]
\centering
\caption{\revision{First-cube measurements for every heuristic reported in \cref{tab:pass_at_k_scores}\rebrev{ and for AlphaMapleSAT}. \emph{Top 10\%} and \emph{Top 25\%}: fraction of first-cube choices whose chosen variable is within the top 10\% and top 25\% of the CNF's variables by clause-occurrence frequency. \emph{Entropy}: mean Shannon entropy of the first-cube variable across 5 runs per benchmark (\rebrev{estimator ceiling} $\log_2 5 \approx 2.32$). \emph{Agreement}: mean pairwise run agreement on the first-cube variable. All rows pool 5 runs over the 100 test benchmarks. (The median rank percentile is 0.000 for all methods.) \rebrev{OS denotes online search (AlphaMapleSAT).}}}
\label{tab:first-cube-comparison}
\vspace{0.4em}
\begin{tabular}{@{}ll rrrr@{}}
\toprule
& \textbf{Heuristic} & \textbf{Top 10\%} & \textbf{Top 25\%} & \textbf{Entropy} & \textbf{Agreement} \\
\midrule
\multirow{4}{*}{\rotatebox[origin=c]{90}{\scriptsize\textit{LLM}}}
& Qwen3-4B (base) & 79.0\% & 87.4\% & 0.55 & 70.4\% \\
& GPT-OSS-120B (teacher) & 89.8\% & 93.8\% & 1.13 & 44.0\% \\
& Qwen3-32B & 88.6\% & 91.6\% & 1.17 & 42.3\% \\
& GPT-OSS-20B & 84.0\% & 89.4\% & 1.29 & 36.2\% \\
\midrule
\multirow{6}{*}{\rotatebox[origin=c]{90}{\scriptsize\textit{Symbolic}}}
& random & 73.6\% & 80.4\% & 2.31 & 0.4\% \\
& \zternary & 85.4\% & 88.0\% & 0.22 & 87.8\% \\
& \zheuleschur & 85.4\% & 88.0\% & 0.24 & 86.7\% \\
& \zunit & 85.6\% & 87.8\% & 0.23 & 87.6\% \\
& \zmarchcu & 85.0\% & 87.2\% & 0.25 & 86.3\% \\
& \zheuleu & 85.2\% & 87.4\% & 0.24 & 86.8\% \\
\midrule
\multirow{1}{*}{\rotatebox[origin=c]{90}{\scriptsize\textit{OS}}}
& AlphaMapleSAT & 71.0\% & 77.0\% & 0.00 & 100.0\% \\
\midrule
\multirow{3}{*}{\rotatebox[origin=c]{90}{\scriptsize\textit{Ours}}}
& Qwen3-4B-DPO & 81.7\% & 89.8\% & 0.59 & 67.8\% \\
& Qwen3-4B-SFT & 79.8\% & 86.0\% & 1.93 & 12.0\% \\
& Qwen3-4B-SFT-DPO & 78.8\% & 85.0\% & 1.78 & 18.1\% \\
\bottomrule
\end{tabular}

\end{table}

\rebrev{\Cref{tab:student-teacher-accounting} complements the aggregate comparison of \cref{tab:first-cube-comparison} at the level of individual benchmarks, with indices as in \cref{fig:success-heatmap-extended} and a benchmark counted as solved when any of the five runs solves it. The base model, the teacher, and Qwen3-4B-SFT have 42 benchmarks in common, and beyond these the base model solves 4 benchmarks and the teacher 5. The student learns 3 of the teacher's benchmarks, retains 2 of the base model's, and solves 4 that neither of them solves. It also fails on 2 benchmarks that the base model solves and 2 that the teacher solves, with one benchmark in both groups. SFT therefore combines knowledge from both models and finds new benchmarks, but it also loses some of what each of them solves. Two factors that we measured plausibly account for the student exceeding its teacher on \passat{5}. The first is throughput, since decision latency bounds how many cubing decisions fit in the 30-minute window. Qwen3-4B-SFT makes 247 decisions per benchmark on average against 143 for the GPT-OSS-120B teacher, and its per-run mean of 46.0 solved benchmarks against 42.8 shows the same advantage within a single run. Qwen3-4B-SFT-DPO is similar, with 243 decisions and a per-run mean of 47.4. The second is the diversity discussed above, which lets five runs of the student cover more benchmarks than five runs of a teacher whose decisions are more strongly correlated across runs.}

\begin{table}[ht]
\centering
\caption{\rebrev{Benchmark-level accounting of the Qwen3-4B-SFT student against the Qwen3-4B base model and the GPT-OSS-120B teacher. A benchmark counts as solved when any of the five runs solves it, and indices follow \cref{fig:success-heatmap-extended}, with 0--49 for SAT and 50--99 for UNSAT. Benchmark 24 is solved by both the base model and the teacher and not by the student, so it appears in both loss rows.}}
\label{tab:student-teacher-accounting}
\begin{tabular}{@{}lrp{0.5\linewidth}@{}}
\toprule
\textbf{Set} & \textbf{Count} & \textbf{Benchmarks} \\
\midrule
Solved by all three models & 42 & 0--11, 13, 14, 16--21, 50--69, 71, 72 \\
Base model, beyond the common set & 4 & 15, 24, 27, 32 \\
Teacher, beyond the common set & 5 & 12, 23, 24, 26, 31 \\
\midrule
Student learns from teacher & 3 & 12, 23, 26 \\
Student retains from base model & 2 & 15, 27 \\
Student solves, neither does & 4 & 22, 25, 73, 74 \\
Base model solves, student does not & 2 & 24, 32 \\
Teacher solves, student does not & 2 & 24, 31 \\
\bottomrule
\end{tabular}
\end{table}

\paragraph{Per-benchmark diversity view.} \revision{\Cref{fig:var-run-agreement} complements the cross-heuristic comparison with a per-benchmark view restricted to Qwen3-4B-SFT-DPO.} \rebrev{On 43 of the 100 benchmarks the most-popular first-cube variable is picked by only 1 of 5 runs, 30 benchmarks have a 2/5 mode, 15 have a 3/5 mode, 6 have a 4/5 mode, and 6 show full 5/5 agreement.}

\begin{figure}[ht]
\centering
\includegraphics[width=\linewidth]{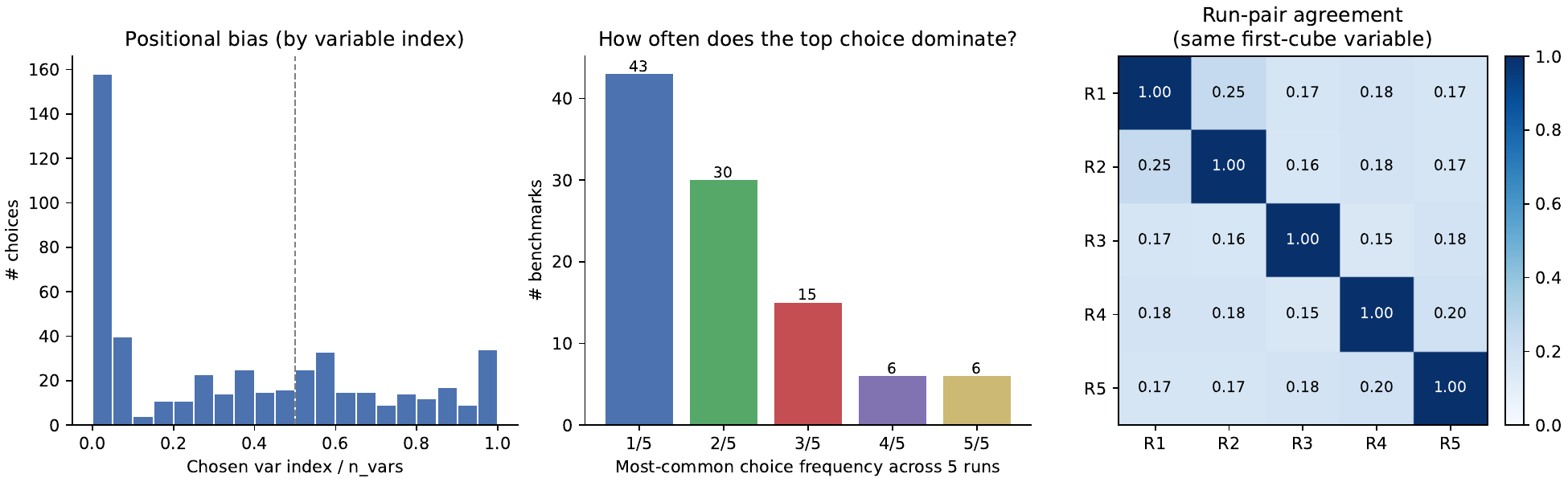}
\caption{\revision{Per-benchmark diversity of the first-cube variable chosen by Qwen3-4B-SFT-DPO across 5 evaluation runs. \emph{Left}: positional bias, i.e. the distribution of chosen variable index divided by the total number of variables in the formula. \emph{Middle}: how often the most-common first-cube variable dominates the 5 runs. On \rebrev{43/100} benchmarks the top choice wins only 1 of 5 runs\rebrev{, and 6 show full 5/5 agreement}. \emph{Right}: pairwise run-agreement matrix (student only), where all off-diagonal entries are \rebrev{15--25\%}, agreeing with the mean 18\% reported in \cref{tab:first-cube-comparison}.}}
\label{fig:var-run-agreement}
\end{figure}

\subsubsection{Do the reasoning traces cluster by strategy or by decoding run?}~\label{subsec:lexical-style}

\revision{The H1--H12 labels in \cref{tab:heuristic-classification} are assigned from reasoning text using an LLM-as-a-judge. To test whether these labels reflect genuine reasoning differences or are an artifact of decoding variability, we ran an unsupervised clustering on 2000 reasoning traces (400 from each of the 5 evaluation runs, balanced over SAT and UNSAT, numerical tokens masked to prevent benchmark-specific leakage). TF-IDF features (unigrams and bigrams) were reduced to 50 dimensions via truncated SVD and clustered with $k$-means for $k \in \{3, \ldots, 10\}$, selecting the best $k$ by silhouette score.}

\rebrev{The best silhouette score, 0.088 at $k=10$, is very weak (silhouette $> 0.5$ typically indicates well-separated clusters), so the reasoning traces do not separate cleanly. The clusters that do emerge (\cref{fig:unsup-clusters}, \cref{tab:unsup-clusters}) are not explained by the decoding run, since each of the five runs contributes between 21\% and 31\% of every cluster, close to the 20\% a uniform split would give. They instead track recurring descriptive vocabulary, for example clause-frequency terms (``high clause'', ``clause frequency''), literal-polarity terms (``positive'', ``negative forms''), and assignment terms (``true'', ``false'', ``setting true''), which recur across all runs rather than in any single one.}

\rebrev{This tempers the interpretation of \cref{tab:heuristic-classification} without reducing it to a decoding artifact. The lexical clusters are weak and spread evenly across decoding runs, so per-trace vocabulary is neither a clean recovery of the H1--H12 categories nor an effect of the decoding seed. Because our main behavioural claim (\cref{subsec:diversity}) rests on realised first-cube variables rather than on reasoning text, it is unaffected in either case.}

\begin{figure}[ht]
\centering
\includegraphics[width=0.62\linewidth]{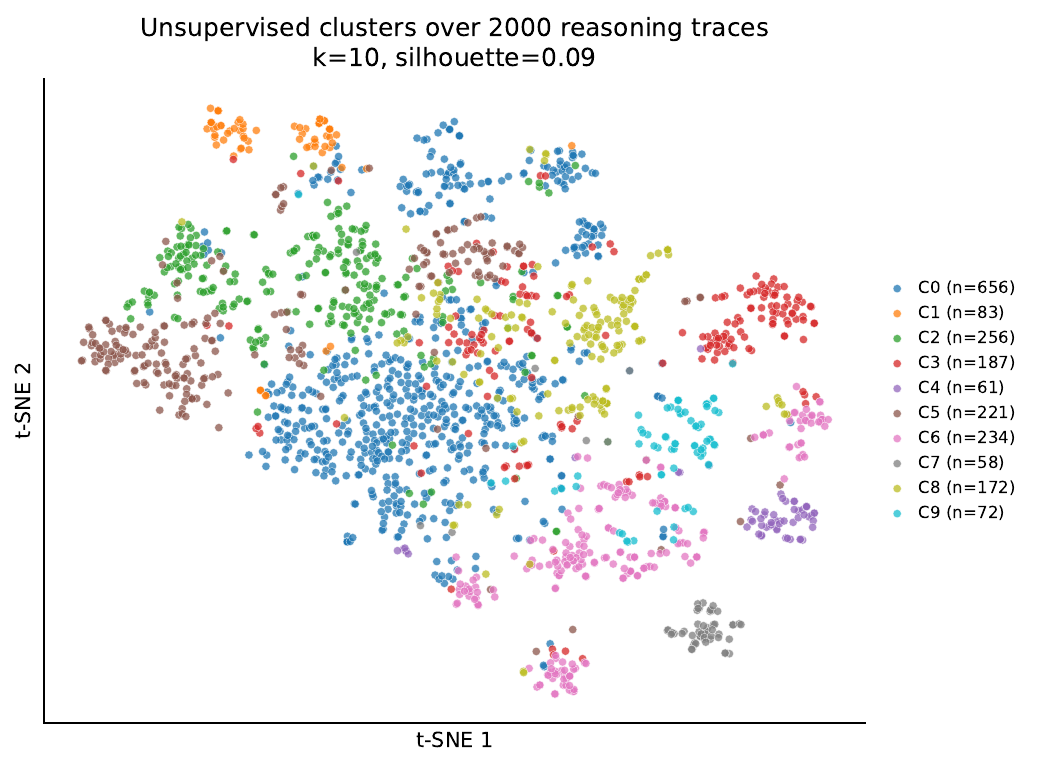}
\caption{\revision{Unsupervised clustering of 2000 reasoning traces (400 per run, SAT/UNSAT balanced, numerical tokens masked). TF-IDF features reduced via truncated SVD, $k$-means with \rebrev{$k{=}10$ chosen by silhouette (score 0.09)}. 2D t-SNE projection coloured by cluster. \Cref{tab:unsup-clusters} lists the per-cluster distinguishing terms.}}
\label{fig:unsup-clusters}
\end{figure}

\begin{table}[ht]
\centering
\small
\caption{\rebrev{Per-cluster size and top eight distinguishing terms (relative to the global mean) for the unsupervised clustering in \cref{fig:unsup-clusters}. No cluster is specific to a decoding run, with each run supplying between 21\% and 31\% of every cluster, so the surface vocabulary recurs across runs rather than tracking a single decoding seed.}}
\label{tab:unsup-clusters}
\vspace{0.4em}
\begin{tabular}{@{}rrp{0.74\linewidth}@{}}
\toprule
\textbf{C} & \textbf{n} & \textbf{Top distinguishing terms (vs.\ global mean)} \\
\midrule
C0 & 656 & num num, cube, choosing, cube num, highly, pairs, variable appears, like \\
C2 & 256 & true, false, true false, assigning true, setting, setting true, assigning, frequent variable \\
C6 & 234 & high clause, variables high, high, clause, clause frequency, likely, variables, frequency \\
C5 & 221 & negative, positive, positive negative, negative forms, forms, negative literals, occurrences positive, forms num \\
C3 & 187 & making, candidate, splitting, variable num, strong candidate, num appears, strong, making strong \\
C8 & 172 & clauses fixing, fixing, fixing simplifies, variable num, num appears, simplifies, appears clauses, simplifies subformulas \\
C1 & 83 & subformulas num, easiest, easiest subformulas, cube reduces, cube, creates easiest, cube creates, num num \\
C9 & 72 & clauses likely, occurrences clauses, likely, occurrences, variables highest, variables frequent, easier subformulas, appearances clauses \\
C4 & 61 & counts, clause counts, counts likely, variables high, high clause, high, clause, likely \\
C7 & 58 & clause involvement, involvement, variables high, high clause, involvement likely, clause, high, variables \\
\bottomrule
\end{tabular}

\end{table}

\subsection{\finalrev{Learned Heuristic Classification of Reasoning Traces}}~\label{subsec:learned-heuristic-analysis}

\finalrev{To understand what strategies our model invokes in its reasoning text, we analyzed 388 output traces (191 SAT, 197 UNSAT) produced by Qwen3-4B-SFT-DPO and classified them into established SAT heuristic categories. An LLM-as-a-judge (Claude-Sonnet-4.5) parsed each trace, which typically contains 1--3 sentences of reasoning followed by a variable decision. The classifier prompt (\cref{subsec:heuristic-classification-prompt}) includes an explicit primary-heuristic selection rule and five worked few-shot examples drawn from traces where two independent human annotators agreed on the primary label. Since traces often exhibit multiple heuristics, the classifier assigns all applicable categories plus a primary label. To validate this automated classification, two human experts independently labeled a random sample of 50 traces. The LLM judge's Cohen's $\kappa$ with each human (0.45 and 0.48) lands in the moderate-agreement range and is slightly above inter-human agreement itself (0.37), indicating that the LLM judge performs within the range of human disagreement (\cref{subsec:inter-rater}).}

\begin{table*}[ht]
\centering
\small
\caption{Heuristic Classification of (388) Reasoning Traces of Qwen3-4B-SFT-DPO}
\label{tab:heuristic-classification}
\begin{tabular}{@{}llrrrl@{}}
\toprule
ID & Heuristic & Count & \% & Primary & Literature \\
\midrule
H1 & Variable Frequency Analysis & 285 & 73.5 & 192 & \cite{moskewicz2001chaff} \\
H2 & Balanced Subformula Splitting & 170 & 43.8 & 60 & \cite{heule2011cube} \\
H6 & Subformula Complexity Reduction & 169 & 43.6 & 37 & \cite{heule2011cube} \\
H7 & Polarity Balance Analysis & 89 & 22.9 & 59 & \cite{silva1996grasp} \\
H5 & Structural Pattern Recognition & 33 & 8.5 & 9 & \cite{ansotegui2019community} \\
H4 & Clause-Size Weighting & 32 & 8.2 & 3 & \cite{jeroslow1990solving} \\
H3 & Clause Elimination Reasoning & 30 & 7.7 & 5 & \cite{davis1962machine} \\
H8 & Critical Clause Identification & 20 & 5.2 & 7 & \cite{freeman1995improvements} \\
H12 & Other & 14 & 3.6 & 11 & — \\
H9 & Unit Propagation Potential & 9 & 2.3 & 5 & \cite{le2001exploiting} \\
H10 & Symmetry Detection & 4 & 1.0 & 0 & \cite{katebi2010symmetry} \\
\bottomrule
\end{tabular}

\end{table*}

\finalrev{The dominant heuristic is \emph{Variable Frequency Analysis}~(H1, 73.5\% of traces, primary in 192 of 388), which counts variable occurrences across clauses. The next most frequent are \emph{Balanced Subformula Splitting}~(H2, 43.8\%) and \emph{Subformula Complexity Reduction}~(H6, 43.6\%), followed by \emph{Polarity Balance}~(H7, 22.9\%). The remaining heuristics are less frequent ($<9\%$), and concrete examples for each are in \cref{subsec:example-traces}. This classification shows that the model's reasoning text invokes a diverse set of established SAT heuristics, with H1 being the default strategy without being explicitly prompted. \Cref{subsec:freq-rank,subsec:lexical-style} provide two caveats on the interpretation of these counts, namely that the H1 top-frequency behavior is shared by every LLM we measured (not training-specific), \rebrev{and that unsupervised clustering of the reasoning text separates it only weakly and not cleanly along these categories, so the counts are best read as coarse strategy descriptions rather than crisp labels}.}

\subsection{Inter-Rater Agreement for Heuristic Classification}~\label{subsec:inter-rater}

\finalrev{To validate the LLM-as-a-judge approach, two human experts (A and B) independently classified the same random sample of 50 reasoning traces into the primary-heuristic categories. We then ran the same classification prompt via Claude-Sonnet-4.5 and computed Cohen's $\kappa$~\citep{cohen1960coefficient} on the primary label and Fleiss' $\kappa$ across all three raters. Per-heuristic counts and Claude-vs-human $\kappa$ for the top-7 categories are reported in the right block of \cref{tab:heuristic-top7} in the main body. The overall pairwise $\kappa$ is 0.45 (Human~A vs Claude) and 0.48 (Human~B vs Claude), both in the \emph{moderate} Landis--Koch range, and slightly above inter-human agreement (0.37) on the same 50 traces. Fleiss' $\kappa$ = 0.43 across all three raters confirms the LLM judge performs within the range of human-to-human disagreement. Per-heuristic $\kappa$ is cleanest on Polarity Balance (H7) and H1 Variable Frequency. The near-zero or negative per-heuristic $\kappa$ for H3 and H5 reflects very low base rates (at most 4 primaries out of 50 on either rater's sheet), where $\kappa$ is numerically unstable, although the raw counts remain close across raters.}

\subsection{Prompts Used}~\label{subsec:prompts-used}

Note that in the following prompts we use the term \emph{cube} to mean the decision variable as well, inspired by the Z3's documentation \citep{de2008z3}.

\subsubsection{System Prompt}~\label{subsec:system-prompt}

We use the following system prompt for the interaction with the different models. The model is expected to respond using two sections: a reasoning section containing analysis, and an answer section with a well-formatted cube. The answer is a parenthesized pair of two opposite integers. For example, the answer $(-1, 1)$ indicates that the splitting variable is $1$, so the subformulas are derived by substituting the variable $1$ as True and False.

\begin{tcolorbox}[
    colback=gray!5!white,
    colframe=gray!75!black,
    fonttitle=\bfseries\small,
    title={System Prompt},
    breakable,
    sharp corners,
    boxrule=0.5pt,
    left=2mm, right=2mm, top=2mm, bottom=2mm,
]
\small
\begin{verbatim}
You are a SAT solving expert focusing on Cube and Conquer algorithm.
You are known to follow unconventional thinking patterns in order to derive the
best cube. Your goal is to suggest the best cube that reduces a given formula
to the easiest subformulas. A cube is a pair of a variable and its negation,
like (v, -v) or (-v, v). You are allowed to respond only in the following format
and do not forget to include all opening and closing XML tags in your response:
<reasoning>
reasoning
</reasoning>
<answer>
answer
</answer>
\end{verbatim}
\end{tcolorbox}

\subsubsection{Task Prompt}~\label{subsec:task-prompt}

The following prompt accompanies each CNF formula input:

\begin{tcolorbox}[
    colback=gray!5!white,
    colframe=gray!75!black,
    fonttitle=\bfseries\small,
    title={Task Prompt},
    breakable,
    sharp corners,
    boxrule=0.5pt,
    left=2mm, right=2mm, top=2mm, bottom=2mm
]
\small
\begin{verbatim}
Read the DIMACS CNF input. The first line is in the format
'p cnf <num_vars> <num_clauses>'. Each subsequent line represents a clause,
ending with a '0' which marks the end of that clause. Clauses consist of boolean
variables, notated as numbers in the range [-num_vars, num_vars]. The cube you
choose will be used as part of the Cube and Conquer algorithm to reduce the
problem to two subformulas. Your task is to read the CNF input and select the
cube that creates the easiest to solve subformulas. For the reasoning respond
with as few sentences as possible and for the answer respond with exactly one
cube, like (v, -v) or (-v, v).
\end{verbatim}
\end{tcolorbox}

\subsubsection{DPO Reasoning Generation Prompt}~\label{subsec:dpo-reasoning-prompt}

For generating reasoning trajectories in our DPO training data, we use the following prompt with the teacher model:

\begin{tcolorbox}[
    colback=gray!5!white,
    colframe=gray!75!black,
    fonttitle=\bfseries\small,
    title={DPO Reasoning Generation Prompt},
    breakable,
    sharp corners,
    boxrule=0.5pt,
    left=2mm, right=2mm, top=2mm, bottom=2mm
]
\small
\begin{verbatim}
You are a SAT solving expert focusing on Cube and Conquer algorithm. You are to
take a CNF formula and a cube and reason about why the cube is a good cube to
consider. Structure your reasoning output as though you have chosen the cube to
be the best cube for formula simplification. Be thorough in your reasoning and
consider all possible implications of the cube, potentially with specific
examples/explanations/reductions that are impactful.
\end{verbatim}
\end{tcolorbox}

\subsubsection{Heuristic Classification Prompt}~\label{subsec:heuristic-classification-prompt}

For classifying the obtained reasoning traces, we use the following prompt:

\begin{tcolorbox}[
    colback=gray!5!white,
    colframe=gray!75!black,
    fonttitle=\bfseries\small,
    title={Heuristic Classification Prompt (v2, few-shot)},
    breakable,
    sharp corners,
    boxrule=0.5pt,
    left=2mm, right=2mm, top=2mm, bottom=2mm
]
\small
\begin{verbatim}
You are a SAT solving expert analyzing reasoning traces from a neural model
trained for Cube-and-Conquer decisions.

## Heuristic Taxonomy
Classify the trace into ONE OR MORE of these established SAT heuristics:

{taxonomy_table}

## How to select the PRIMARY heuristic

The primary heuristic is the ONE category that best explains *why THIS
specific variable was chosen* in the trace's own reasoning. Use these rules:

1. Do not default to H1. H1 (Variable Frequency) applies only when the
   trace's stated reason is raw frequency ("appears in many clauses",
   "most occurrences"). If the trace mentions frequency but frames its
   decision around another mechanism (balance, polarity, simplification,
   critical clauses), the primary is that other mechanism.

2. Prefer the more specific category when two apply.
   - H1 vs H2: H2 wins if the trace mentions "balanced", "evenly",
     "two equal subformulas", "split the formula in half".
   - H1 vs H6: H6 wins if the trace's main argument is "simplifies /
     reduces complexity of subformulas" without naming frequency as the
     reason.
   - H1 vs H7: H7 wins if the trace emphasises appearing in both positive
     AND negative polarity (not just total frequency).
   - H1 vs H8: H8 wins if the trace names clauses as "critical",
     "important", or "key" rather than treating the variable itself as
     the unit of analysis.
   - H1 vs H9: H9 wins if the trace mentions unit propagation, cascading
     implications, or forced assignments.

3. For multiple primary candidates, pick the one the trace states FIRST
   as its reasoning. The primary is the first-mentioned mechanism, not
   the most general one.

## Worked examples (calibration)

These five examples show the correct PRIMARY label for traces whose
reasoning is representative of each category. Two independent human
annotators agreed on each primary label.

Example 1 (primary: H1 - Variable Frequency)
Trace: "Variables with high clause frequencies are likely to be pivotal.
47 appears in many clauses, so fixing it simplifies the formula
significantly."
Correct primary: H1.
Why: the trace's stated mechanism is raw frequency. "Simplifies the
formula" is a consequence, not the reason.

Example 2 (primary: H2 - Balanced Subformula Splitting)
Trace: "The cube that reduces the CNF input to the easiest subformulas is
(-36, 36). This choice splits the problem into two balanced and simple
subformulas."
Correct primary: H2.
Why: the trace explicitly states "balanced" and "two subformulas" as the
mechanism. Frequency is not mentioned.

Example 3 (primary: H6 - Subformula Complexity Reduction)
Trace: "The cube that reduces the problem to the easiest subformulas is
(10, -10)."
Correct primary: H6.
Why: the only claimed mechanism is simplification of the resulting
subformulas. No frequency, no balance, no polarity.

Example 4 (primary: H7 - Polarity Balance Analysis)
Trace: "The most balanced variable in the CNF is -143, which appears in
20 clauses as both a positive and negative literal. Assigning it to true
or false will evenly split the formula into two subformulas that are
likely to be easy to solve."
Correct primary: H7.
Why: the first and stated reason is "both a positive and negative
literal" (polarity balance). H2 and H1 are secondary.

Example 5 (primary: H8 - Critical Clause Identification)
Trace: "The most impactful variable to eliminate is 107, as it appears in
multiple critical clauses and its negation is not present in many of
them. Assigning 107 to true reduces a large number of complex clauses,
simplifying the problem significantly."
Correct primary: H8.
Why: the trace names "critical clauses" as the mechanism.

## Reasoning Trace to Classify:
```
{trace}
```

## Task:
Identify which heuristics from the taxonomy are demonstrated in this
reasoning trace, and select the PRIMARY heuristic following the selection
rules above.

IMPORTANT:
- A trace may use multiple heuristics (list all that apply).
- The PRIMARY must be the single category that best explains why THIS
  variable was chosen (see rules).
- Provide brief quote evidence for each heuristic identified.
- If something doesn't fit the taxonomy, use H12 (Other) and describe it.

Respond with JSON only:
{{
    "heuristics": ["H1", "H4"],
    "primary_heuristic": "H1",
    "evidence": {{
        "H1": "brief quote showing this heuristic",
        "H4": "brief quote"
    }},
    "reasoning_quality": "high|medium|low",
    "notes": "optional observations"
}}
\end{verbatim}
\end{tcolorbox}

\subsection{Reasoning Traces and Heuristic Classification Examples}~\label{subsec:example-traces}

We present five example reasoning traces generated by our trained model (Qwen3-4B-SFT-DPO) during the cubing phase, alongside their classification from Claude-Sonnet-4. These traces contain all of the heuristic categories and have diverse quality (from high to low) according to Claude-Sonnet-4. The mapping of heuristics to examples is: H1 (ex. 2, 3), H2 (ex. 1, 2), H3 (ex. 3), H4 (ex. 4), H5 (ex. 1), H6 (ex. 2), H7 (ex. 3), H8 (ex. 2), H9 (ex. 4), H10 (ex. 1), H12 (ex. 5).

\begin{tcolorbox}[
    colback=gray!5!white,
    colframe=green!75!black,
    fonttitle=\bfseries\small,
    title={Reasoning Trace and Heuristic Classification Example 1},
    breakable,
    sharp corners,
    boxrule=0.5pt,
    left=2mm, right=2mm, top=2mm, bottom=2mm
]
\small
\begin{verbatim}
Reasoning Trace:

<reasoning>
The CNF is composed of a large number of clauses, each with only two variables
negated. The structure is highly symmetric and forms complete bipartite graphs
between groups of 144 variables. Each clause connects one variable from the
first group to one from the second group in a fixed pattern. The cube that
creates the easiest subformulas is the  one that splits the graph evenly,
removing half the nodes from each side and preserving as many connections as
possible. This results in two nearly identical subgraphs with minimal disruption
to the original symmetry.
</reasoning>

<answer>
(9, -9)
</answer>

=========================================

Heuristic Classification Analysis:

{
"heuristics": ["H10", "H2", "H5"],

"primary_heuristic": "H10",

"evidence": {
   "H10": "The structure is highly symmetric and forms complete bipartite graphs
           between groups of 144 variables... preserving as many connections as
           possible... minimal disruption to the original symmetry",

   "H2": "The cube that creates the easiest subformulas is the one that splits
          the graph evenly, removing half the nodes from each side...
          This results in two nearly identical subgraphs",

   "H5": "The CNF is composed of a large number of clauses...
          The structure is highly symmetric and forms complete bipartite
          graphs between groups of 144 variables"
},

"reasoning_quality": "high",

"notes": "The trace demonstrates sophisticated structural analysis, recognizing
          the bipartite graph structure and symmetry in the formula.
          The primary focus is on exploiting symmetry (H10) to create balanced
          splits (H2). The reasoning shows deep understanding of how graph
          structure relates to SAT solving efficiency.",
}
\end{verbatim}
\end{tcolorbox}

\begin{tcolorbox}[
    colback=gray!5!white,
    colframe=green!75!black,
    fonttitle=\bfseries\small,
    title={Reasoning Trace and Heuristic Classification Example 2},
    breakable,
    sharp corners,
    boxrule=0.5pt,
    left=2mm, right=2mm, top=2mm, bottom=2mm
]
\small
\begin{verbatim}
Reasoning Trace:

<reasoning>
Variable 103 appears in many clauses and is part of several critical
subformulas. Fixing it early simplifies both branches significantly,
creating two easy subformulas.
</reasoning>

<answer>
(103, -103)
</answer>

=========================================

Heuristic Classification Analysis:

{
"heuristics": ["H1", "H8", "H6", "H2"],

"primary_heuristic": "H1",

"evidence": {
   "H1": "Variable 103 appears in many clauses",
   "H8": "is part of several critical subformulas",
   "H6": "Fixing it early simplifies both branches significantly",
   "H2": "creating two easy subformulas"
},

"reasoning_quality": "medium",

"notes": "The trace combines frequency analysis (H1) with critical clause
          identification (H8) and complexity reduction reasoning (H6).
          The mention of 'two easy subformulas' suggests balanced
          splitting (H2), though the reasoning is fairly generic and
          lacks specific quantitative analysis."
}
\end{verbatim}
\end{tcolorbox}

\begin{tcolorbox}[
    colback=gray!5!white,
    colframe=green!75!black,
    fonttitle=\bfseries\small,
    title={Reasoning Trace and Heuristic Classification Example 3},
    breakable,
    sharp corners,
    boxrule=0.5pt,
    left=2mm, right=2mm, top=2mm, bottom=2mm
]
\small
\begin{verbatim}
Reasoning Trace:

<reasoning>
The most frequent variable is 219, appearing in multiple clauses with both
positive and negative signs. Fixing it to true or false would eliminate many
of its occurrences from the problem, reducing complexity significantly.
</reasoning>

<answer>
(219, -219)
</answer>

=========================================

Heuristic Classification Analysis:

{
"heuristics": ["H1", "H3", "H7"],

"primary_heuristic": "H1",

"evidence": {
   "H1": "The most frequent variable is 219, appearing in multiple clauses",
   "H3": "Fixing it to true or false would eliminate many of its occurrences
          from the problem",
   "H7": "appearing in multiple clauses with both positive and negative signs"
},

"reasoning_quality": "medium",

"notes": "The trace primarily uses frequency analysis (H1) to identify variable
          219 as most frequent. It also mentions polarity balance (H7) by noting
          the variable appears in both positive and negative forms, and
          considers clause elimination (H3) as a consequence of the assignment.
          The reasoning is straightforward but lacks depth in analyzing the
          actual impact on formula structure."
}
\end{verbatim}
\end{tcolorbox}

\begin{tcolorbox}[
    colback=gray!5!white,
    colframe=green!75!black,
    fonttitle=\bfseries\small,
    title={Reasoning Trace and Heuristic Classification Example 4},
    breakable,
    sharp corners,
    boxrule=0.5pt,
    left=2mm, right=2mm, top=2mm, bottom=2mm
]
\small
\begin{verbatim}
Reasoning Trace:

<reasoning>
Most clauses have two variables. Clauses with one variable (e.g., -52 0) are
likely to be the easiest to split on, as they only involve a single variable
and not many others.
</reasoning>

<answer>
(52, -52)
</answer>

=========================================

Heuristic Classification Analysis:

{
"heuristics": ["H4", "H9"],

"primary_heuristic": "H4",

"evidence": {
   "H4": "Most clauses have two variables. Clauses with one variable
         (e.g., -52 0) are likely to be the easiest to split on",
   "H9": "Clauses with one variable... are likely to be the easiest to
          split on, as they only involve a single variable"
},

"reasoning_quality": "low",

"notes": "The reasoning identifies unit clauses (single-variable clauses) as
          easy splitting targets, which relates to clause-size weighting
          (shorter clauses are more constraining). The logic is somewhat
          confused - unit clauses force assignments rather than being good
          split candidates. The trace shows basic awareness of clause size
          importance but misapplies the concept. Choosing variable 52 from
          a unit clause '-52 0' would immediately satisfy/falsify that clause
          via unit propagation."
}
\end{verbatim}
\end{tcolorbox}

\begin{tcolorbox}[
    colback=gray!5!white,
    colframe=green!75!black,
    fonttitle=\bfseries\small,
    title={Reasoning Trace and Heuristic Classification Example 5},
    breakable,
    sharp corners,
    boxrule=0.5pt,
    left=2mm, right=2mm, top=2mm, bottom=2mm
]
\small
\begin{verbatim}
Reasoning Trace:

<reasoning>
-40 99 -149 -19 -106 0
-45 -95 -88 -102 -22 0
</reasoning>

<answer>
(106, -106)
</answer>

=========================================

Heuristic Classification Analysis:

{
"heuristics": ["H12"],

"primary_heuristic": "H12",

"evidence": {
   "H12": "The trace shows only raw clause-like data '-40 99 -149 -19 -106 0'
           and '-45 -95 -88 -102 -22 0' with no explicit reasoning about
           variable selection"
},

"reasoning_quality": "low",

"notes": "This trace appears to be malformed or represents a compressed/encoded
          format rather than explicit reasoning. The output shows two lines of
          literals terminated by 0 (DIMACS-style clause format), and the answer
          selects variable 106. The selection of 106 may relate to it appearing
          in the first clause as -106, but no reasoning process is articulated.
          This could potentially be H1 (frequency analysis) or H3 (clause
          elimination) but there's insufficient evidence to classify properly."
}
\end{verbatim}
\end{tcolorbox}

\end{document}